\documentclass{article}
\usepackage{iclr2027_conference,times}
\usepackage[T1]{fontenc}
\usepackage{amsmath,amssymb,amsthm,mathtools}
\usepackage{mathrsfs}
\usepackage{graphicx,booktabs,array}
\usepackage{hyperref}
\usepackage{xurl}
\hypersetup{hidelinks,pdfauthor={Dongxu Li; Yinuo Zhang; Hongyu Zhang; Feng Tian},
  pdftitle={Generalized Residual Closure: General Learning Dynamics for Stability-Plasticity Compatibility},
  pdfsubject={arXiv preprint: general learning dynamics}}
\newtheorem{theorem}{Theorem}[section]
\newtheorem{proposition}[theorem]{Proposition}
\newtheorem{lemma}[theorem]{Lemma}
\newtheorem{corollary}[theorem]{Corollary}
\theoremstyle{definition}

\newtheorem{example}[theorem]{Example}
\theoremstyle{remark}

\newcommand{\R}{\mathbb{R}}
\newcommand{\Hist}{\mathcal{X}}

\newcommand{\Learner}{\mathcal{M}}
\newcommand{\Residual}{\mathfrak{R}}
\newcommand{\Old}{\mathcal{S}_{\mathrm{old}}}

\newcommand{\Kplast}{K_{\mathrm{plast}}}

\DeclareMathOperator{\Def}{Def}
\DeclareMathOperator{\Close}{Close}
\DeclareMathOperator{\Open}{Open}
\DeclareMathOperator{\Compile}{Compile}

\DeclareMathOperator{\im}{im}
\DeclareMathOperator*{\argmin}{arg\,min}

\title{Generalized Residual Closure:\\
General Learning Dynamics for\\
Stability--Plasticity Compatibility}
\author{Dongxu Li$^{1,*}$ \quad Yinuo Zhang$^{2}$ \quad
Hongyu Zhang$^{1}$ \quad Feng Tian$^{1}$\\
\normalfont $^{1}$Xi'an Jiaotong University\\
\normalfont $^{2}$Xi'an University of Architecture and Technology\\
\normalfont $^{*}$Corresponding author: \texttt{skey@stu.xjtu.edu.cn}}
\iclrfinalcopy
\begin{document}
\maketitle
\fancyhead{}
\begin{abstract}
Learning must acquire new capabilities while preserving both prior responsibilities and the capacity to learn again. We introduce Generalized Residual Closure (GRC), a framework for learning as recursive closure of future-relevant discrepancies: closing a residual establishes the conditions for subsequent prediction, interaction, and learning. Under a complete representation--relation description at a fixed learner--world boundary, persistent internal learning has two primitive modes: Transformation within a representation and revision of the Representation itself. We establish a local tangent decomposition under regularity assumptions and a criterion for when representation revision is necessary. In an affine model, we derive a necessary-and-sufficient condition for stability--plasticity compatibility and the unique solution of a constrained quadratic update problem, which preserves registered old responsibilities while reducing residuals with an effective safe response. We prove that reconstructive semantic protection weakly enlarges the safe-response operator relative to preserving an exact historical realization. Dynamic sufficiency and future-closure viability extend representation adequacy from current prediction to lawful future updating and continued learning. Growth Learning expands the lawful closure domain or lowers optimal closure cost without regression of the registered capability--cost frontier; a conditional commit rule maintains this order. Restricted-sector recoveries and a conditional representation theorem connect the framework to optimization, machine learning, and control. Together, these results organize adaptation, representation revision, and reusable capability within a common account of continued learning.

\end{abstract}
\section{Introduction: What Should a Residual Change?}
\label{sec:introduction}
Machine learning usually begins after the coordinates of adaptation have
been chosen. Neural networks update parameters; reinforcement learning
updates values, policies, or models on a declared state space; adapters
restrict writes to prescribed subspaces; architecture search explores a
structural grammar. Optimization determines how these selected objects
should change. Their selection raises a prior question: \emph{what should
 a residual change?}

A residual does not carry its own write address. A prediction error may
indicate a wrong relation within an adequate representation, aliasing of
future-distinguishable histories, unavailable evidence, or a family that
cannot express the required continuation. Treating every discrepancy as a
parameter-update instruction conflates evidence of insufficiency with a
decision about where it resides.

Generalized Residual Closure (GRC) describes learning through this
distinction. A represented state opens into relations that explain the
present and predict continuation. Interaction exposes a contract-relative
discrepancy, and a lawful persistent revision establishes the conditions
for subsequent predictions, actions, and learning. Closure is recursive:
the successor may expose a new residual, or remain sufficient under later
interaction. At a fixed boundary, Transformation changes relations within
the current abstraction, while Representation Learning changes which
states must be distinguished. A joint revision may contain both.

Stability and plasticity then share a constraint geometry. Stability
protects reconstructible future-operational responsibilities; plasticity
is the useful response that remains lawful under those constraints.
Present success remains insufficient for continual learning. Dynamic
sufficiency preserves information needed for future updating, and
future-closure viability retains lawful continuation paths. Growth Learning
additionally compares the domain and optimal cost of those paths.

The paper makes four contributions. First, it specifies the recursive
semantics and least-lawful closure formulation. Second, it gives a local
two-mode decomposition, a conditional finite classification, and a
representation-obstruction criterion. Third, it derives affine
stability--plasticity compatibility and reconstructive semantic dominance,
then distinguishes dynamic sufficiency from future viability. Fourth, it
defines certified growth, develops restricted correspondences with
established methods, and supplies a conditional GRC representation theorem.
The results organize these requirements jointly; their assumptions do not
establish universal algorithmic superiority.

\paragraph{Related work.}
Prior work has separately defined state through future predictions or
behavioral equivalence \citep{littman2001predictive,ferns2011bisimulation},
constrained continual updates to protect previous capabilities
\citep{lopezpaz2017gem,farajtabar2020ogd}, studied the loss and preservation
of plasticity in nonstationary learning
\citep{dohare2024plasticity,han2026fire,qiu2026splitlora}, and enlarged the
object of learning through meta-learning and architecture search
\citep{finn2017maml,liu2019darts}. GRC does not claim these components
individually as new; its question is upstream: which representational
level owns a residual, what must remain invariant during closure, and
what future learning must remain possible afterward.
Appendix~\ref{app:related-work} details the comparisons and scope.

\section{Recursive Generalized Residual Closure}
\label{sec:principles}
Let $M\in\Learner$ contain all persistent learner variables relevant to future responses and updates. Let $x\in\Hist$ be a detailed interaction state or history. The contract $\Xi$ registers admissible contexts, actions, evidence, information rights, tolerances, and resources; $\partial$ fixes the learner--world boundary. GRC imposes three structural requirements: $U^\star$ assigns an operational owner to every required distinction; $I^\star$ identifies states through registered future behavior; and $L^\star$ seeks a least-cost lawful completion. Their satisfaction requires the relevant existence and sufficiency conditions.

\subsection{Future-Operational Representation}
For a fixed learner $M$, define
\begin{equation}
 x\sim_{\Xi,M}y\ \Longleftrightarrow\
 \mathrm{Resp}_M(c,x)\equiv_\Xi\mathrm{Resp}_M(c,y)
 \quad\forall c\in\mathrm{Ctx}_\Xi.
 \label{eq:operational-identity}
\end{equation}
Here $\equiv_\Xi$ is a response equivalence; quantitative approximation is treated separately. The \emph{Kernel} is the quotient $K_M=\Hist/{\sim_{\Xi,M}}$, with map $C_M:x\mapsto[x]$ and current state $k=C_Mx$. It identifies detailed states under a fixed learner, not different learner implementations.

Defining state through future behavior has close precedents in
predictive-state representations and bisimulation-based state abstraction
\citep{littman2001predictive,ferns2011bisimulation}. GRC makes compatibility
with registered learner-update events an explicit admissibility requirement,
formalized as dynamic sufficiency in Section~\ref{sec:growth}.

The \emph{Ring} $\mathcal R_\Xi(K_M)$ comprises admissible relational realizations on that quotient, with active relations $\rho_M\in\mathcal R_\Xi(K_M)$. The quotient alone does not uniquely select these relations.

A representation is unchanged when $C_N=\phi\circ C_M$ for an admissible isomorphism $\phi$. Write $\varpi_\Xi(M)=[C_M]$ for this abstraction class. Preserving it permits changes to learned relations. Stability separately protects the still-valid responsibilities registered by $\Xi_{\mathrm{old}}$.

\subsection{Explanation, Prediction, and Residual Formation}
Opening exposes active relations; for an admissible event $e$, they predict $\widehat k^+=\widehat U_{M,e}(C_Mx)$. Actual interaction produces $U_e(x)$. Where both continuations are defined on the current quotient, their discrepancy is
\begin{equation}
 \Residual_{\Xi,e}(M,x)=\Def_\Xi\bigl(C_MU_e(x),
                      \widehat U_{M,e}C_M(x)\bigr).
 \label{eq:generalized-residual}
\end{equation}
The prediction $\widehat U_{M,e}$ is distinct from an exactly induced update. If the quotient cannot support an adequate continuation, that obstruction must itself be represented rather than evaluated through an undefined map. The residual diagnoses insufficiency without selecting its remedy.

\begin{figure}[t]
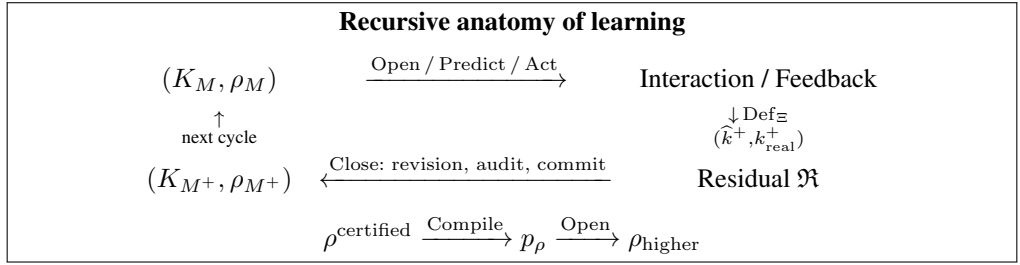

\centering
\fbox{\begin{minipage}{0.94\linewidth}
\centering
\textbf{Recursive anatomy of learning}\par\medskip
\(\begin{array}{ccc}
 (K_M,\rho_M)
 &\xrightarrow{\mathrm{Open\,/\,Predict\,/\,Act}}
 &\text{Interaction / Feedback}\\[0.4em]
 \substack{\uparrow\\\text{next cycle}}
 &&\substack{\downarrow\,\Def_\Xi\\(\widehat k^+,k_{\mathrm{real}}^+)}\\[0.4em]
 (K_{M^+},\rho_{M^+})
 &\xleftarrow{\mathrm{Close:\ revision,\ audit,\ commit}}
 &\text{Residual }\Residual
\end{array}\)\par\medskip
\(\rho^{\mathrm{certified}}\xrightarrow{\Compile}
       p_\rho\xrightarrow{\Open}\rho_{\mathrm{higher}}\)
\end{minipage}}
\caption{Closure establishes a successor learner for continued interaction. A Transformation retains the quotient while revising relations. Compilation separately turns validated reusable relations into primitives. The next interaction may expose another residual or remain sufficient.}
\label{fig:recursive-anatomy}
\end{figure}

\subsection{Least-Lawful Closure}
Let $\operatorname{Adm}_\Xi(M,r)$ enforce registered hard obligations and let $C_{\mathrm{full}}$ charge the declared search, revision, verification, deployment, and maintenance costs. The completion specification is
\begin{equation}
 M^+\in\underset{N\in\operatorname{Adm}_\Xi(M,r)}{\argmin}
 C_{\mathrm{full}}(N\mid M)
 \quad\text{subject to}\quad
 d_\Xi\bigl(\Residual_\Xi(N)\bigr)\le\varepsilon_\Xi.
 \label{eq:master-close}
\end{equation}
The notation presupposes an attained lawful minimum; it does not construct a solver. Temporary reasoning or retrieval becomes persistent learning when its consequences change future-relevant learner state. The successor may retain its quotient and change relations, or revise both. Prediction proposes an unrealized continuation; Compile installs a validated reusable primitive. Closure can support either further learning or sufficient ongoing operation with zero new residual.

\subsection{Same-Representation Closure Load}
\label{sec:grade}
Let $\mathcal C_T^\Xi(M)$ contain the residual corrections achievable by
a single admissible edit within the current representation and unit
declared cost. For a nested dilation family $D_\lambda$, define
\begin{equation}
 q_T(r\mid M)=\inf\{\lambda>0:
       r\in D_\lambda\mathcal C_T^\Xi(M)\}.
 \label{eq:closure-load}
\end{equation}
All protected duties constrain that same edit. Under the calibrated
threshold-to-feasibility conditions of Appendix~\ref{app:grade},
$q_T>1$ identifies a limitation of the declared closure family.
It does not by itself establish representation insufficiency: missing
evidence, transport, and incomplete search or solving must be distinguished.
Boundary feasibility depends on attainment of the infimum. The proposed
closure-grade extension and its remaining conditions are kept in
Appendix~\ref{app:grade-filtration}.

\section{The Two Primitive Modes of Learning}
\label{sec:modes-grade}
\label{sec:modes}

A residual identifies a mismatch, but does not determine which internal
object should change. Fix the contract $\Xi$ and learner--world boundary
$\partial$, and write $\varpi_\Xi(M)=[C_M]$ for the representation class
defined in Section~\ref{sec:principles}. It records the distinctions made
between detailed states, up to admissible relabeling. The current state
$k=C_M(x)$, the representation $C_M$, and the protected old responsibilities
$\Old(M)$ are different objects: a state can change without changing its
representation, and relations can change while the representation remains fixed.

\begin{theorem}[T1: Two-mode learning factorization]
\label{thm:two-modes}
On a finite-dimensional smooth stratum with fixed $\Xi$ and $\partial$, suppose
$\varpi_\Xi$ is smooth and has locally constant rank. Set
$V_M=\ker D\varpi_\Xi(M)$. Then
\begin{equation}
0\longrightarrow V_M\longrightarrow T_M\Learner_{\Xi}^{(\partial)}
\xrightarrow{D\varpi_\Xi(M)}\im D\varpi_\Xi(M)\longrightarrow0
\label{eq:two-mode-sequence}
\end{equation}
is exact. After choosing a local splitting with complementary space $H_M$,
each direction has a unique decomposition
\begin{equation}
d=d_T+d_R,\qquad d_T\in V_M,\quad d_R\in H_M.
\label{eq:two-mode-decomposition}
\end{equation}
The components depend on the splitting; whether $D\varpi_\Xi(M)[d]=0$ does not.
\end{theorem}

\paragraph{Proof sketch and finite transitions.}
The derivative's kernel gives fixed-representation tangent motion. A chosen
right splitting supplies $d_R$, and $d_T=d-d_R$ lies in that kernel.
Joint learning can have both components. The full proof is in
Appendix~\ref{app:modes}. Finite transitions are classified by whether
$C_{M^+}\cong C_M$; their exhaustion of relevant persistent changes uses
representation--relation completeness, Assumption A3 in
Appendix~\ref{app:contracts}. This conditional classification does not give
a global smooth splitting across singular strata. A realization change
that preserves the abstraction is not automatically Representation learning.

\paragraph{When representation revision is necessary.}
Let $\mathfrak R(C)$ be the exhaustive class of admissible relation laws on a
fixed representation, and let $F$ be the required future response. Define
\begin{equation}
\Omega_C(F)=\inf_{f\in\mathfrak R(C)}D_\Xi(F,f\circ C).
\label{eq:representation-obstruction}
\end{equation}
\begin{proposition}[P2: Representation necessity]
\label{prop:representation-necessity}
With the boundary, evidence, relation class, and tolerance fixed, if
$\Omega_C(F)>\varepsilon_\Xi$, no relation-only update satisfies the contract.
Any successful internal closure must leave that representation class.
\end{proposition}

The proof and a Hilbert-space decomposition are in
Appendix~\ref{app:representation-necessity}. The obstruction concerns the
whole admissible class; optimizer failure alone cannot establish it.
It does not guarantee a feasible revised representation, and an infimum
at the tolerance boundary is not itself an attained solution.

Querying, retrieval, sensing, and tool use can supply evidence without
persistent internal learning. Incorporating their results changes relations,
representation, or both. If the interface machinery itself is learned,
the enlarged persistent state and boundary must be declared; classification
still depends on whether $C_M$ actually changes. Interface modification is
not automatically a third mode or a Representation transition.

\section{Stability and Plasticity from One Geometry}
\label{sec:dynamics-compatibility}

GRC treats stability as a constraint on lawful change and plasticity as the
residual response remaining within that constraint. The protected object is
the registered old responsibility, rather than every future response;
learning may introduce new behavior while preserving what remains required.
The following affine sector makes this compatibility condition explicit.

\paragraph{Semantic-safe response geometry.}
Let $d\in\R^n$ be an edit and $\Old(M)\in\R^p$ record the old responsibilities
declared by $\Xi_{\mathrm{old}}$. The matrix $A\in\R^{p\times n}$ encodes their
change, and $B\in\R^{m\times n}$ maps edits to reductions of a current
residual $r\in\R^m$. For this section's exact model,
\begin{equation}
\Old(M+d)-\Old(M)=Ad,\qquad r(M+d)=r-Bd.
\label{eq:affine-duty-model}
\end{equation}
Choose $H\succ0$ with edit energy $d^\top Hd$. The safe space is $\ker A$.
Its metric-dependent safe inverse and residual-response operator are
\begin{align}
P_H&=H^{-1}-H^{-1}A^\top(AH^{-1}A^\top)^\dagger AH^{-1},
\label{eq:safe-inverse}\\
\Kplast&=BP_HB^\top\succeq0.
\label{eq:safe-plasticity}
\end{align}
Here $P_H=P_H^\top\succeq0$, $AP_H=0$, and $\im P_H=\ker A$. It is a
safe inverse, not generally an idempotent Euclidean projector.

\begin{theorem}[T2: Stability--plasticity compatibility]
\label{thm:compatibility}
Under the affine model~\eqref{eq:affine-duty-model} with $H\succ0$, the following
conditions are equivalent:
\begin{align}
&\exists d\in\ker A:\ \langle r,Bd\rangle>0,\label{eq:compatible-direction}\\
&P_HB^\top r\ne0,\label{eq:compatible-covector}\\
&r^\top\Kplast r>0.\label{eq:compatible-response}
\end{align}
Thus old-responsibility preservation and a strict first-order reduction of
squared residual error are compatible precisely when the residual has an
effective safe response.
\end{theorem}

\paragraph{Proof sketch.}
If $P_HB^\top r\ne0$, choose $d=P_HB^\top r$. Then $Ad=0$ and
$\langle r,Bd\rangle=(B^\top r)^\top P_H(B^\top r)>0$.
Conversely, $P_HB^\top r=0$ places $B^\top r$ in $\im A^\top$, which
annihilates every safe direction. The full proof is in
Appendix~\ref{app:compatibility}. The result characterizes Transformation
directions only when the edit coordinates fix the representation;
otherwise it applies to the joint edit space.

\paragraph{A constructive affine step.}
For $\eta>0$, consider the strictly convex subproblem
\begin{equation}
\min_{Ad=0}\ \frac12\|r-Bd\|^2+\frac{1}{2\eta}d^\top Hd.
\label{eq:proximal-safe-step}
\end{equation}
Its unique solution and resulting residual are
\begin{align}
d^\star&=\eta P_HB^\top(I+\eta\Kplast)^{-1}r,
\label{eq:proximal-safe-solution}\\
r^+&=(I+\eta\Kplast)^{-1}r.
\label{eq:proximal-safe-residual}
\end{align}
The step satisfies $Ad^\star=0$ and $\|r^+\|\le\|r\|$, with strict inequality
exactly when $\Kplast r\ne0$. The derivation appears in
Appendix~\ref{app:proximal-step}. This affine subproblem need not eliminate
the residual and is not the complete $\Close$ operation.

\paragraph{Reconstructive protection.}
Preserving old operational responsibilities can allow more edits than freezing
their historical implementation. Let $T_{\rm exact}\subseteq T_{\rm rec}$ be
the corresponding safe spaces. With the same $B$ and $H$, define
$P_T^H=H^{-1/2}\Pi_{H^{1/2}T}H^{-1/2}$ and $K_T=BP_T^HB^\top$.
\begin{corollary}[Reconstructive semantic stability dominance]
\label{cor:reconstructive-dominance}
If $T_{\rm exact}\subseteq T_{\rm rec}$, then
$P_{T_{\rm rec}}^H-P_{T_{\rm exact}}^H\succeq0$ and
\begin{equation}
K_{\rm rec}-K_{\rm exact}
=B(P_{T_{\rm rec}}^H-P_{T_{\rm exact}}^H)B^\top\succeq0.
\label{eq:reconstructive-dominance}
\end{equation}
\end{corollary}
Whitening by $H^{1/2}$ reduces the claim to nested orthogonal projectors;
Appendix~\ref{app:reconstructive-dominance} gives the full proof. Strict gain
for the current residual requires $r^\top(K_{\rm rec}-K_{\rm exact})r>0$.
Operator dominance alone gives neither universal strict benefit nor a
uniformly smaller post-update Euclidean residual.

This geometry is related to continual-learning methods that restrict
updates to protect previous behavior, including gradient-projection
approaches \citep{farajtabar2020ogd} and recent explicit treatments of the
stability--plasticity tradeoff \citep{han2026fire,qiu2026splitlora}.
The distinction here is that the protected object is future-operational
semantics rather than a particular historical parameter or gradient subspace.

\begin{figure}[t]
\centering
\setlength{\unitlength}{0.8pt}
\begin{picture}(425,90)
 \put(30,12){\line(1,0){200}}
 \put(230,12){\line(1,1){50}}
 \put(80,62){\line(1,0){200}}
 \put(30,12){\line(1,1){50}}
 \put(198,75){\makebox(0,0){$T_{\rm rec}$}}
 {\linethickness{0.9pt}\put(55,37){\line(1,0){200}}}
 \put(76,45){\makebox(0,0)[l]{$T_{\rm exact}$}}
 \put(153,37){\circle*{4}}
 \put(150,27){\makebox(0,0){$0$}}
 \put(153,37){\vector(1,1){20}}
 \put(180,52){\makebox(0,0)[l]{$d_{\rm rec}$}}
 \put(292,43){\makebox(130,0)[l]{$T_{\rm exact}\subseteq T_{\rm rec}$}}
 \put(292,24){\makebox(130,0)[l]{same $H$ and $B$}}
\end{picture}
\caption{Stability is a geometry of lawful change. A line and a patch of a
plane illustrate nested safe subspaces; the drawing is conceptual, not to
scale. Reconstructive protection permits additional directions when historical
realization constraints are unnecessary. Strict response gain still requires
those directions to affect the residual through $B$.}
\label{fig:stability-geometry}
\end{figure}
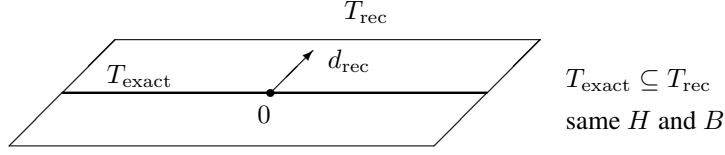

\paragraph{Nonlinear scope.}
For a smooth nonlinear learner, $A=D\Old(M)$ and $B$ is the local
residual-response Jacobian. Tangent safety alone does not preserve a finite
nonlinear endpoint; actual old responsibilities must be checked before
commit. Nor does a missing local response exclude a nonlocal path within
the same representation. Global Representation necessity requires the
separate obstruction in Proposition~\ref{prop:representation-necessity}.

\section{Dynamic Sufficiency, Continual Learning, and Growth}
\label{sec:growth}
Preserving the past and learning the present need not preserve the ability
to learn again. A compression may support the current prediction yet discard
information needed to determine the correct update after future evidence.
For a candidate representation $C_t:\Hist_t\to Z_t$ and registered detailed
update $U_e:\Hist_t\to\Hist_{t+1}$, update sufficiency requires
\begin{equation}
 C_{t+1}U_e\simeq_\Xi\bar U_eC_t.
 \label{eq:update-sufficiency}
\end{equation}
Here $\simeq_\Xi$ is exact operational equivalence; approximation requires a
separate error contract. The induced updater $\bar U_e$ is distinct from
the prediction $\widehat U_e$. Equal compressed inputs cannot determine
different required post-update states, even if their current predictions
agree. Appendix~\ref{app:update-factorization} gives the factorization
criterion and a two-state counterexample. Dynamic sufficiency additionally
requires the corresponding commutation conditions for registered
background changes and internal transitions.

\subsection{Future-Closure Viability}
Let $\Gamma_\Xi(M,r)$ contain the finite lawful paths that close a registered
future demand $r\in\mathcal Y_\Xi^{\rm fut}$. Define
\begin{equation}
 \mathcal V_\Xi(M)=\{r\in\mathcal Y_\Xi^{\rm fut}:
                          \Gamma_\Xi(M,r)\ne\varnothing\}.
 \label{eq:future-viability}
\end{equation}
Viability on a declared class $\mathcal Y_0$ means
$\mathcal Y_0\subseteq\mathcal V_\Xi(M)$. A certificate such as
$B_fP_HB_f^\top\succeq\kappa I$, $\kappa>0$, concerns only the registered
local probe family. It neither presumes knowledge of unknown future tasks
nor replaces multi-step path viability. Appendix~\ref{future-viability}
distinguishes these objects and shows how acquisition can exhaust a
particular future response.

\subsection{Closure Capability}
For finite, nonnegative path costs under a common accounting convention, set
\begin{equation}
 \begin{split}
 c^\star_\Xi(r\mid M)&=\inf_{\gamma\in\Gamma_\Xi(M,r)}C_{\rm full}(\gamma),
 \qquad \inf\varnothing=+\infty,\\
 \mathscr F_\Xi(M)&=\{(r,c)\in\mathcal Y_\Xi^{\rm fut}\times\R_{\ge0}:
                         c^\star_\Xi(r\mid M)\le c\}.
 \end{split}
 \label{eq:growth-frontier}
\end{equation}
This frontier is an optimal-cost epigraph; its infimum need not be attained.
Write $M'\succeq_\Xi M$ when old registered semantics remain reconstructible
and $\mathscr F_\Xi(M')\supseteq\mathscr F_\Xi(M)$. Strict growth
$M'\succ_\Xi M$ additionally expands the viable domain or strictly lowers
the optimal cost of a previously closable demand. Thus growth means a
broader or cheaper lawful future, with no regression elsewhere in the
registered comparison. Information rights, responsibilities, and cost
accounting must be aligned.

\subsection{Certified Persistent Growth}
Let $\widetilde M_t=\mathsf T_{t\to t+1}M_t$ be a lawful no-learning transport
baseline in the next background. A trial may regress; a persistent commit
requires both a sound audit and certified dominance over that baseline.

\begin{theorem}[T3: Certified Monotone Causal Growth]
\label{thm:growth}
Assume the transport baseline is lawful and the audit certifies current
closure, semantic preservation, dynamic sufficiency, declared future
viability, and resource and safety obligations. Commit a candidate only
if the audit passes and the candidate dominates $\widetilde M_t$;
otherwise retain $\widetilde M_t$. Then
\[
 M_{t+1}\succeq_{\Xi_{t+1}}\widetilde M_t.
\]
If an accepted candidate also strictly expands the closure domain or
lowers an optimal closure cost, then
$M_{t+1}\succ_{\Xi_{t+1}}\widetilde M_t$.
\end{theorem}
The two branches of the rule give the result by dominance or reflexivity;
Appendix~\ref{app:growth-proof} supplies the full statement and proof.
The theorem does not construct an audit, a dominance certificate, or an
improving candidate. Rejection retains the lawful baseline without
claiming current closure or refunding trial costs. The restricted discovery
result in Appendix~\ref{app:growth} needs additional evidence, reachability,
and fair terminating search; it does not imply unlimited growth.

Carry can enlarge the domain of lawful future closure; Compile can turn
recurrent certified relations into reusable primitives with lower future
cost. These are two possible growth effects, not automatic consequences
of either operation. Frontier preservation and an actual expansion or
optimal-cost improvement remain required. Appendix~\ref{app:growth-compilation}
states the compilation, reopening, and full-cost conditions. Successful
closure can thereby become reusable capacity for subsequent learning.

\section{Existing Paradigms as Restricted GRC Sectors}
\label{sec:sectors}
Established methods fix different parts of the closure problem. GRC can
retain their solvers while reconsidering those choices. An exact sector
embedding must preserve the declared inputs, outputs, information rights,
randomness, and updates; a similar functional form alone is insufficient.

\subsection{Optimization}
Fix the representation, boundary, program, and writable coordinates.
For $g=\nabla L(\theta)$, $\eta>0$, and $H\succ0$,
\begin{equation}
 \underset{d}{\arg\min}
 \left\{g^\top d+\frac{1}{2\eta}d^\top Hd\right\}
 =-\eta H^{-1}g.
 \label{eq:sector-metric-step}
\end{equation}
Thus $H=I$ yields ordinary gradient descent, while positive-definite
Fisher geometry recovers a natural-gradient sector
\citep{amari1998natural}; suitable positive-definite local curvature
choices produce Newton-like sectors. This algebra
determines how chosen coordinates move. GRC additionally asks whether
those coordinates can express the required repair.

\subsection{Deep and Residual Learning}
A differentiable family
$\mathcal F_{\mathcal P}=\{M_\theta:\theta\in\Theta\}$ fixes an outer
machine grammar. Parameter changes can alter relations and latent
representations, so they are not automatically Transformation learning.
For a residual block \citep{he2016resnet},
\[
 h_{\ell+1}=h_\ell+F_\ell(h_\ell),
\]
the forward program fixes the residual owner and the additive correction
form. This is a
forward-map correspondence, not a proof of lawful closure or an embedding
of the full training dynamics. Completion can reopen the owner, correction
form, or machine family.

\subsection{RL and Control}
Classical reinforcement learning and dynamic programming formulate value,
policy, and Bellman updates on a declared state--action representation
\citep{sutton2018rl}. With a future-sufficient state abstraction held fixed,
lawful persistent updates that preserve $[C_M]$ are Transformation closures.
For a sufficient control abstraction $s=C(x)$, with specified actions,
transitions, and costs, $\mathcal TV-V$ is a Bellman consistency residual
\citep{bellman1952dynamic}. Updates preserving the registered encoder
class are Transformation steps. If a proposed $C$ instead merges detailed
histories with inequivalent required continuations, a law depending only
on $s$ cannot represent both. This state-aliasing obstruction concerns
the abstraction, rather than the optimizer
\citep{givan2003equivalence,littman2001predictive}. Recovering a particular control
or RL algorithm also requires its sampling, update, and feasibility semantics.

\subsection{Containment and Completion}
A completion preserves lawful incumbent realizations while reopening
selected choices: $\mathcal F_A\subseteq\mathcal F_{A^+}$. Under a common
objective, information rights, constraints, and complete cost accounting,
\[
 J_{A^+}^\star\le J_A^\star,
 \qquad J_A^\star=\inf_{M\in\mathcal F_A}J_\Xi(M).
\]
Strict improvement requires a lawful witness outside $\mathcal F_A$ with
lower objective, including search, audit, compilation, maintenance, and
deployment costs. Inclusion does not guarantee discovery or runtime
advantage. Appendix~\ref{app:sectors} provides the derivations;
Appendix~\ref{app:sector-extended} discusses further structural
correspondences and their differing proof status.

\begin{table}[t]
\caption{Representative restricted choices and candidate GRC revisions.
Rows describe conditional correspondences, not universal algorithmic
equivalence or guaranteed improvement.}
\label{tab:sectors}
\centering
\small
\begin{tabular}{@{}>{\raggedright\arraybackslash}p{0.19\linewidth}
                  >{\raggedright\arraybackslash}p{0.34\linewidth}
                  >{\raggedright\arraybackslash}p{0.39\linewidth}@{}}
\toprule
Paradigm & Predeclared structure & Candidate GRC revision\\
\midrule
Optimization & Coordinates and metric & Representation or metric\\
Deep learning & Differentiable family & Outer machine family\\
ResNet & Interface and additive form & Owner or correction form\\
LoRA/adapters \citep{hu2022lora} & Factorized edit family & Rank, family, or owner\\
RAG \citep{lewis2020rag} & Evidence-access boundary & Separate missing evidence from internal failure\\
RL/control & State/action abstraction & State representation\\
Meta-learning/NAS \citep{finn2017maml,liu2019darts} & Meta/search family & The family itself\\
\bottomrule
\end{tabular}
\end{table}

\section{Conditional GRC Representation and Intelligence}
\label{sec:representation}
\subsection{A conditional GRC representation}
\label{sec:conditional-representation}

Can an abstract learner admit the preceding operational description?
Let $\mathcal X$ contain detailed interaction states or histories, and let
$M\in\Learner$ be its complete persistent state. Under a fixed contract
and learner--world boundary, registered responses define
\begin{equation}
\begin{aligned}
x\sim_{\Xi,M}y
&\iff \operatorname{Resp}_M(c,x)\equiv_\Xi\operatorname{Resp}_M(c,y)
\quad\forall c\in\operatorname{Ctx}_\Xi,\\
K_M&=\mathcal X/{\sim_{\Xi,M}},\qquad C_M(x)=[x]_{\sim_{\Xi,M}}.
\end{aligned}
\label{eq:history-kernel-quotient}
\end{equation}
The response identity must induce an equivalence relation for each fixed
$M$. Thus $K_M$ identifies detailed states, not entire learners.
Write $\varpi_\Xi(M)=[C_M]$ for the encoder's class up to admissible
relabeling; relation learning can change responses without changing this
partition.

\begin{theorem}[T4: Conditional GRC representation]
\label{thm:representation}
Under a fixed $(\Xi,\partial)$, suppose a persistent learner satisfies
R1--R8 of Appendix~\ref{app:representation}: state completeness, future
identity, conditional dynamic sufficiency, nontrivial discrepancy response,
representation revisability, future viability, fail-closed persistence,
and minimum complete-cost selection. On the events admitting the required
quotient-level laws, its accepted resolving transitions have an
operationally equivalent description through $K_M$, relation laws $\rho_M$,
and generalized residuals. A transition either preserves $[C_M]$
(Transformation) or changes it (Representation), possibly with simultaneous
relation changes. For the learner's inherited lawful resolving family,
\begin{equation}
M^+\in\argmin_{N\in\operatorname{Law}_\Xi(M,r)}C_{\rm full}(N\mid M).
\label{eq:representation-lawful-selection}
\end{equation}
\end{theorem}

\paragraph{Proof sketch.}
Future identity gives the quotient; dynamic sufficiency permits continuation
laws to descend. Comparing the current law's prediction with realized
continuation gives a residual. The endpoint partition test gives the T/R
classification, and the inherited acceptance semantics and R8 give lawful
minimum-cost selection. Appendix~\ref{app:representation} supplies the
construction. Failure of descent is recorded as an obstruction, not a
missing map supplied by the theorem. The prediction can be wrong even when
an exact induced continuation exists.

The conclusion preserves registered responses and accepted transition
semantics. It does not establish a unique implementation, efficient
discovery, or superiority over other sectors. R8 includes an actual selected
minimizer; candidate admission alone guarantees neither search success nor
an attained optimum.

\paragraph{Independent capitalization property.}
R9 additionally permits repeatedly certified, recurrent relations to be
compiled into reopenable primitives $p_\rho$, while preserving semantics,
dynamic sufficiency, and future viability with lower declared complete
future cost. This permits recursive Open/Compile operation when the
eligibility conditions hold. A primitive is distinct from the entire
quotient $K_M$; the two operations are not algebraic inverses. Neither R9
nor strict capability growth follows from R1--R8.

\subsection{An operational view of intelligence}
\label{sec:intelligence-profile}
GRC suggests interpreting intelligence as the recursive capacity to turn
unresolved future-relevant distinctions into reusable closure capability
without exhausting the ability to learn again. The capability--cost profile
of Section~\ref{sec:growth} gives this interpretation a declared operational
scope: which future demands remain lawfully closable, and at what complete
cost.

Reasoning performs temporary closure search; learning installs a certified
persistent successor; compiled memory retains a reopenable realization of
acquired capability. Repeatedly useful relations can thereby become
primitives supporting later relations. These are distinct outcomes: a
successful trial need not be a persistent commit, and a commit need not
qualify for compilation. The account is a contract-relative interpretation,
not a unique or universal scalar definition of intelligence.
Appendix~\ref{app:intelligence} develops its scope and failure tests.

\section{Conclusion and Scope}
\label{sec:limitations}
GRC connects three questions that jointly determine continued learning:
what a residual requires the learner to change, which acquired
responsibilities must survive that change, and what future learning remains
possible afterward. The two-mode analysis distinguishes adaptation within
a representation from revision of the representation itself. The affine
compatibility theorem gives an exact criterion and a constructive update
for preserving old responsibilities while learning from a new residual.
Reconstructive protection weakly enlarges the safe-response operator
under nested safe spaces and a common edit metric and residual map. Dynamic
sufficiency and the closure-capability frontier extend the account to
future updating and resource-aware growth.

The results have explicit domains. The two-mode classification assumes a
fixed boundary and a complete representation--relation description, with
regularity needed for the tangent decomposition. Compatibility is exact
in the affine sector; finite nonlinear updates require endpoint audit.
Growth is measured under registered future demands and common full-cost
accounting, with sound certification and discovery treated as separate
obligations. The representation theorem characterizes learners satisfying
its operational and selection assumptions. Efficient discovery and
performance comparisons remain separate questions; the counterexamples
in Appendix~\ref{app:counterexamples} delimit these results.

The resulting learning cycle is recursive: a successful closure establishes
a new condition for prediction and interaction, and subsequent evidence
can expose further residuals. When semantic, future-viability, and full-cost
conditions justify compilation, recurrent closure becomes reusable
capacity. GRC thus places preservation, revision, and growth within one
operational account of how learning makes further learning possible.

\label{maintext:end}
\subsection*{AI/LLM use statement}
AI/LLM tools were used only for writing assistance and language polishing.
The mathematical source material, theorem arguments, and proofs were
supplied by the author. AI/LLMs were not used to generate new mathematical
results or supply missing proofs.
\subsection*{Reproducibility statement}
This paper develops a theoretical framework without reporting computational
experiments. Appendix~\ref{app:contracts} collects the semantic and geometric
assumptions; Appendices~\ref{app:modes}, \ref{app:compatibility},
\ref{app:growth}, and \ref{app:representation} give the arguments for
T1--T4. Appendix~\ref{app:counterexamples} records counterexamples and
impossibility regimes. Conditions still needed for broader results are
identified explicitly, including the grade construction in
Appendix~\ref{app:grade} and unrestricted certification and discovery.

\bibliographystyle{iclr2027_conference}
\bibliography{references,references_additions}
\clearpage
\appendix
\section{Assumptions and Formal Semantic Setting}
\label{app:contracts}
The assumptions below are used selectively. Each result requires its own declared subset; they are not simultaneous universal axioms for every learner.

\subsection{Contract, state, and boundary}
The operational contract $\Xi$ registers queries, actions, feedback, update events, evidence rights, interventions, tolerances, resources, and safety constraints. Equivalence, stability, sufficiency, viability, and growth are relative to this contract. A detailed interaction state or history is $x\in\Hist$; the persistent learner state is $M\in\Learner$.

\paragraph{A0: Persistent-state completeness.}
\label{asm:persistent-completeness}
Every internal variable whose persistent value can affect a registered response or later learning transition belongs to $M$. Temporary reasoning traces, speculative candidates, and computational workspace need not belong to $M$ unless they become persistent. An exhaustive classification requires this state description to include the relevant update variables.

\paragraph{A1: Fixed learner--world boundary.}
\label{asm:fixed-boundary}
The internal-transition classification fixes the declared interface $\partial_{t+1}=\partial_t=\partial$. The world can still change and the learner can observe, act, and receive evidence through that interface. A query or measurement that changes available evidence without changing persistent learner state is an external resolution action. If sensing or action machinery becomes a persistent learning variable, it must be included in an enlarged state description. Its transition is classified there by its actual effect on the representation and relations; changing interface machinery alone does not establish a changed state partition.

\subsection{Future identity and representation}
For each fixed learner $M$, registered contexts define the response equivalence of Equation~\eqref{eq:operational-identity}.

\paragraph{A2: Future-operational congruence.}
\label{asm:congruence}
The relation $\sim_{\Xi,M}$ is a well-defined equivalence on $\Hist$. Each operation claimed to descend to the quotient additionally respects that equivalence on its declared domain. Thus
\[
 K_M=\Hist/{\sim_{\Xi,M}},\qquad C_M:\Hist\to K_M,\qquad k=C_Mx.
\]
The quotient space, the quotient map, and a represented state are distinct. A tolerance discrepancy is not automatically a transitive equivalence relation.

Encoders describe the same abstraction when $C_N=\phi\circ C_M$ for an admissible isomorphism $\phi:K_M\to K_N$. Write $C_N\cong C_M$ and $\varpi_\Xi(M)=[C_M]$. This removes mere relabeling of the same partition; it does not identify entire learners by equality of every response. Old responsibilities are separately registered by $\Xi_{\mathrm{old}}$ and, in the affine analysis, encoded by $\Old(M)$.

\paragraph{A3: Representation--relation completeness.}
\label{asm:representation-relation-complete}
For a fixed encoder, let $\operatorname{Rel}_\Xi(C_M)$ contain its admissible relation laws and let $\rho_M$ be the active law. This is the encoder-indexed description of the Ring used in Section~\ref{sec:principles}, with the same admissibility conditions. Every persistent internal change relevant to the claimed complete description must be represented by a relation-law change at fixed abstraction, a change of abstraction, or both. This is an exhaustiveness assumption about the chosen description. It is not a consequence of the elementary kernel--image exact sequence and is not an empirical assertion about an arbitrary parameterization.

\subsection{Regular strata and lawful transitions}
Discrete structural changes may be organized by a stratified learner space
\[
 \Learner_\Xi^{(\partial)}=\bigsqcup_{\sigma\in\Sigma_\Xi}\mathcal S_{\Xi,\sigma}.
\]
A stratum can fix a discrete ownership structure, memory topology, program support, or grammar. A change of stratum need not change the quotient partition.

\paragraph{A4: Local regularity.}
\label{asm:local-regularity}
Where the differential factorization is invoked, $M$ belongs to a finite-dimensional regular smooth stratum and the smooth map $\varpi_\Xi:\mathcal S_{\Xi,\sigma}\to\mathfrak B_\Xi$ has locally constant rank. The horizontal/vertical decomposition is unique only after a complementary space or connection is chosen. This structure is additional to the existence of a future-behavior quotient. Cross-stratum events are compared through their endpoint representations; no global tangent decomposition is implied.

For current residual $r$, $\operatorname{Adm}_\Xi(M,r)$ enforces the hard obligations relevant to a transition claim: old-semantic reconstruction, dynamic sufficiency, future viability, information rights, resource and safety conditions, and current closure where it is required. In Equation~\eqref{eq:master-close}, closure is displayed separately. The resolving set $\operatorname{Law}_\Xi(M,r)$ of T4 includes closure to tolerance. These obligations are not mutually tradable soft penalties.

\paragraph{A5: A lawful baseline when invoked.}
\label{asm:lawful-baseline}
A comparison that relies on fallback must contain a fallback satisfying that comparison's declared obligations, often the incumbent or a lawful transport. An available incumbent is not automatically lawful. In particular, one that leaves the present residual unresolved need not belong to a resolving successor set. T3's rejection branch retains a transported baseline without asserting current closure. Nonemptiness of a resolving set, attainment of a least-cost solution, and availability of a fallback are separate requirements.

\subsection{Costs, closure load, and diagnosis}
The declared complete cost $C_{\rm full}(N\mid M)$ may charge editing, search, audit, deployment, maintenance, and compilation. Its resource scope and accounting origin must be fixed, including how unsuccessful trials are charged.

\paragraph{A6: Comparable costs and information.}
\label{asm:comparable-cost}
Feasible-family comparisons use common information rights, output semantics, initialization conventions, objectives, and resource accounting. Additional search or audit cannot be treated as free overhead. A comparison of future continuation costs need not be a comparison of total lifecycle costs.

\paragraph{A7: Identifiable closure-load scale.}
\label{asm:scalar-identifiability}
A scalar closure-load comparison requires a declared family with
\[
 \lambda_1\le\lambda_2\ \Longrightarrow\ D_{\lambda_1}\mathcal C
 \subseteq D_{\lambda_2}\mathcal C.
\]
The scale and accessible set must be identified sufficiently for the claimed comparison. A Minkowski-gauge interpretation additionally requires the corresponding radial scaling and origin/regularity conditions. Nonconvexity does not prevent writing an infimum, but it prevents importing convex-gauge properties without their hypotheses. If only bounds or set-valued information are available, the output can be an interval or a scoped certificate such as INSIDE, CROSSES, OUTSIDE, or UNIDENTIFIED. The meaning and reliability of these certificates must be specified; their names are not proofs of reachability. Appendix~\ref{app:grade} records the relevant distinctions.

\paragraph{A8: No direct overflow-to-Carry inference.}
\label{asm:no-overflow-carry}
Excessive same-representation load does not identify a representation obstruction. Diagnosis must distinguish saturation of the declared response family, transport mismatch, missing evidence or boundary limitations, incomplete search, and a certified expressive obstruction. This is a logical distinction, not a mandatory fixed routing algorithm. Representation necessity requires a class-wide certificate such as Proposition~\ref{prop:representation-necessity}; it does not follow from a numerical solver failure.

\subsection{Result-specific assumptions}
The following map identifies the additional conditions used in the present results.
\begin{itemize}
\item T1 uses a complete persistent state and a fixed boundary for its learning interpretation, with A4 for the regular differential setting. Its stronger global interpretation invokes A3; the exact sequence alone is a local linear-algebra construction.
\item P2 fixes the representation, evidence, exhaustive admissible relation class, discrepancy, and tolerance. A strict infimal obstruction implies necessity; equality at the tolerance boundary requires an attainment argument for a converse.
\item T2 uses the exact affine response model, a common edit for all old responsibilities and the new residual, and $H\succ0$. With nonlinear Jacobians it describes first-order behavior; finite-endpoint preservation is an additional requirement.
\item Finite future-probe certificates concern only their registered probe family and information rights. Unavailable future-task labels cannot be used to certify current viability.
\item T3 assumes sound audit, a declared capability order and dominance gate, consistent full-cost accounting, and a lawful comparison transport when the background changes. It does not establish search liveness.
\item T4 uses its separately enumerated R1--R8 and the stipulated existence of quotient-level relations on their applicability domain. Compilation requires the additional R9 property. The extensional construction alone supplies neither the regularity of A4 nor an executable full-state encoding.
\end{itemize}

\section{Explain--Predict Semantics}
\label{app:semantics}

\subsection{Opening and Explanation}
For a fixed learner $M$, the Kernel $K_M=\Hist/{\sim_{\Xi,M}}$ is a quotient
state space, $C_M$ is its encoding map, and $k=C_Mx$ is a current state.
An admissible opening is a correspondence
\begin{equation}
 \Open_{\Xi,M}(k)\subseteq\mathcal R_\Xi(K_M),
 \qquad \rho\in\Open_{\Xi,M}(k).
 \label{eq:semantics-opening}
\end{equation}
Here the Ring comprises relational realizations admitted for the fixed
encoding and its operational contract. The source's notation
$\operatorname{Rel}_\Xi(C_M)$ emphasizes that same dependence; identifying
the two descriptions requires the same admissibility semantics. The quotient
alone does not select a unique relation law, and the correspondence does not
assert that a suitable opening always exists.

Explanation means operationally reopening the represented state into
relations through which it acts, predicts, controls, or returns. It need
not be an explanation in natural language. A selected opening supplies
the active realization used for the following prediction.

\subsection{Prediction and Its Comparison Space}
For event $e$, the selected realization predicts a continuation state,
\begin{equation}
 \widehat k_{t+1}=\operatorname{Pred}_{M_t,e}(\rho_t).
 \label{eq:semantics-prediction}
\end{equation}
This is the constraint that the active relation places on an unrealized
future. It is broader than scalar-target or next-token prediction. Where
a selected opening and prediction rule compose into a state-level map,
we write $\widehat k_{t+1}=\widehat U_{M_t,e}(k_t)$. The hat denotes the
learner's prediction, not an exactly induced update.

Prediction and evaluation must share a declared comparison space, denoted
here by $Z_{t,e}$. Its declaration includes any admissible comparison or
transport needed after the event. The notation $\widehat k_{t+1}$ refers to
an element of this space; it does not predict a new quotient $K_{M_{t+1}}$
or presuppose that a successor learner has already been accepted.

\subsection{Realized Continuation Before Learning}
Let
\begin{equation}
 U_e:\mathcal X_t\longrightarrow\mathcal X_{t+1}^{-},
 \qquad x_{t+1}^{-}=U_e(x_t),
 \label{eq:semantics-reality}
\end{equation}
be the detailed continuation generated by the event before a persistent
learning update. If the current representation has a lawful evaluation on
that domain, declare
\begin{equation}
 C_{M_t}^{e}:\mathcal X_{t+1}^{-}\longrightarrow Z_{t,e},
 \qquad k_{t+1}^{\rm real}=C_{M_t}^{e}(x_{t+1}^{-}).
 \label{eq:semantics-evaluation}
\end{equation}
The superscript $e$ records evaluation under the event's declared domain
or transport; it does not denote a newly learned encoder. In the fixed-domain
case, $Z_{t,e}=K_{M_t}$ and $C_{M_t}^{e}=C_{M_t}$ recover the expression in
Section~\ref{sec:principles}.

Such an evaluation is not automatic. If a missing future distinction prevents
a lawful quotient-level continuation, record a representation obstruction.
An undefined evaluation is neither a numerical defect nor evidence of zero
residual. The contract must specify how that obstruction is witnessed.

\subsection{The Pre-learning Residual}
When both values are well typed in the comparison space, define
\begin{equation}
 \Residual_{t,e}^{\Xi}
 =\Def_\Xi\bigl(k_{t+1}^{\rm real},\widehat k_{t+1}\bigr).
 \label{eq:semantics-observed-residual}
\end{equation}
The defect object need not be a scalar loss. Its operational meaning and
closure tolerance are part of the contract. This sequence separates
explanation, prediction, realized continuation, and the observed residual.

A residual diagnoses insufficiency without prescribing its own write
address. In particular, $\Residual\ne0$ does not by itself require a
parameter update. Resolution may involve relation revision, representation
revision, acquiring admissible evidence, lawful transport, or further
search. Boundary actions and temporary search must still be distinguished
from persistent learner-internal changes.

\subsection{Update Commutation Is a Separate Requirement}
For declared source and target encoders
$C_t:\mathcal X_t\to K_t$ and
$C_{t+1}:\mathcal X_{t+1}\to K_{t+1}$, update sufficiency asks whether an
induced updater $\bar U_e:K_t\to K_{t+1}$ exists with
\begin{equation}
 C_{t+1}\circ U_e\simeq_\Xi\bar U_e\circ C_t.
 \label{eq:semantics-update-commutation}
\end{equation}
Here $K_t$ and $K_{t+1}$ denote the quotients associated with those declared
encoders. Their detailed update domain must be specified; it is not silently
identified with the pre-learning domain in Equation~\eqref{eq:semantics-reality}.
For a typed update diagram, the source writes its operator-level diagnostic as
\begin{equation}
 \Residual_{e}^{\Xi,\mathrm{upd}}
 =\Def_\Xi\bigl(C_{t+1}\circ U_e,\bar U_e\circ C_t\bigr).
 \label{eq:semantics-update-defect}
\end{equation}
This diagnostic additionally requires a declared comparison of maps.
If $\bar U_e$ is already the exactly induced updater, the diagram commutes
and its exact defect is zero. If no such updater exists, that failure is
an obstruction to update sufficiency, not a defined zero value of the
expression. Approximate commutation requires its own error contract.

Thus $\widehat U_{M,e}$ makes a prediction that may fail, whereas
$\bar U_e$ denotes a compatible induced update when one exists. The
pre-learning residual diagnoses an observed failure using the current
representation. Commutation tests whether declared representations support
correct updating; it does not define today's learning signal through an
unaccepted future representation.

\subsection{Prediction and Compilation Have Different Roles}
Prediction constrains an unrealized continuation. Compilation instead makes
a repeatedly validated, reopenable relation available as a persistent
primitive after successful use. Writing $p_\rho$ for that primitive,
\begin{equation}
 \rho^{\rm certified}\xrightarrow{\Compile}p_\rho.
 \label{eq:semantics-compilation}
\end{equation}
The symbol denotes the installed primitive, not a predicted world state
or an entire quotient space. Its installation and subsequent reopening
must satisfy the relevant representation, sufficiency, and cost contract.
No identity between Predict and Compile is assumed, and an accepted closure
need not already qualify for compilation.

\subsection{Recursion and Continued Sufficient Operation}
The cycle begins with $k_t=C_{M_t}(x_t)$ and active relations $\rho_t$.
Prediction and realized continuation supply the comparison in
Equation~\eqref{eq:semantics-observed-residual}. Diagnosis and lawful
closure may then produce a persistent successor $M_{t+1}$ with its own
encoding and active relations. For a declared successor detailed state,
\begin{equation}
 k_{t+1}=C_{M_{t+1}}(x_{t+1}).
 \label{eq:semantics-successor-state}
\end{equation}
The encoding acts on $x_{t+1}$, not on the learner $M_{t+1}$.
Any writeback, state transport, or domain change from $x_{t+1}^{-}$ to
$x_{t+1}$ must be declared separately.

The successor supplies the conditions for subsequent opening and prediction.
A Transformation may preserve the quotient up to an admissible isomorphism
while changing the active relation law. Closure therefore need not change
the state partition, and repeated operation need not expose another defect:
the next residual can remain zero. Compilation is a separate qualified
transition, not a mandatory stage after every closure.

\paragraph{Scope of the semantic construction.}
P1 specifies typed operational semantics. It does not establish general
opening existence, background transport, obstruction detection, or
quantitative multi-step approximate sufficiency. Stronger claims require
additional definitions and proofs.

\section{Two-Mode Factorization and Representation Necessity}
\label{app:modes}

\subsection{Local representation projection}
Fix the contract $\Xi$ and learner--world boundary $\partial$. The persistent
learner-state space may be stratified as
\[
\Learner_\Xi^{(\partial)}
=\bigsqcup_{\sigma\in\Sigma_\Xi}\mathcal S_{\Xi,\sigma},
\qquad
\varpi_\Xi(M)=[C_M].
\]
The representation projection records the abstraction class, rather than the
current value $k=C_M(x)$ or a particular engineering realization. For the
local statement, $M$ lies in a finite-dimensional regular stratum on which
$\varpi_\Xi$ is smooth and has locally constant rank. Tangent spaces below
refer to that stratum. In particular,
$V_M=\ker D\varpi_\Xi(M)$ consists of directions that preserve the
representation class to first order. Completeness of the persistent-state
description and the representation--relation description are the declared
assumptions, including Assumption A3 in Appendix~\ref{app:contracts}.

\subsection{Proof of Theorem~\ref{thm:two-modes}}
\begin{proof}
Write $L=D\varpi_\Xi(M)$ and $V_M=\ker L$. The inclusion of $V_M$ into the
tangent space is injective, and the quotient by $V_M$ is naturally isomorphic
to $\im L$. This gives the exact sequence~\eqref{eq:two-mode-sequence}.

Choose a right splitting
$s_M:\im L\to T_M\Learner_{\Xi}^{(\partial)}$ with $L\circ s_M=I$, and let
$H_M=\im s_M$. For an arbitrary tangent direction $d$, set
\[
d_R=s_M(Ld),\qquad d_T=d-d_R.
\]
Then $Ld_T=Ld-Ls_M(Ld)=0$, so $d_T\in V_M$, while $d_R\in H_M$ by
construction. Since $V_M\cap H_M=\{0\}$, the decomposition is unique for the
chosen splitting. The horizontal representative depends on that choice;
whether $Ld=0$ does not.
\end{proof}

Local constant rank supports a regular family of vertical spaces; the
pointwise exact sequence itself is the standard kernel--image construction.
Neither the argument nor the finite-jump classification selects a canonical
connection or establishes a global smooth decomposition across strata.

\paragraph{Interpretation and joint learning.}
A Transformation component satisfies $D\varpi_\Xi(M)[d_T]=0$. It may change
weights, policies, transitions, routing, controllers, or other relation
variables, provided that it preserves the abstraction to first order.
A Representation component changes the abstraction when
$D\varpi_\Xi(M)[d_R]\ne0$. Examples can include state splitting, memory-state
creation, or a change of the representational grammar; their classification
depends on their effect on $C_M$, not on the implementation name.
An event can have both $d_T\ne0$ and $d_R\ne0$. This is a joint event, not a
third primitive component or a partition of algorithms into mutually
exclusive families. The chosen horizontal vector is not canonical, whereas
the test $D\varpi_\Xi(M)[d]=0$ is independent of the splitting.

\subsection{A fixed-representation Hilbert-space sector}
\label{app:hilbert-obstruction}
Suppose future responses lie in a Hilbert space $\mathcal H$ and the responses
expressible under $C$ form a closed linear subspace $\mathcal A_C$. Let
$\Pi_C$ be its orthogonal projector and $Q_C=I-\Pi_C$. For a target $u$ and a
current response $\widehat u\in\mathcal A_C$, the decomposition is
\[
u-\widehat u
=\underbrace{\Pi_Cu-\widehat u}_{r_{\rm rel}}
+\underbrace{Q_Cu}_{r_{\rm rep}}.
\]
For any $\widetilde u\in\mathcal A_C$, orthogonality gives
\[
\|u-\widetilde u\|^2
=\|Q_Cu\|^2+\|\Pi_Cu-\widetilde u\|^2,
\qquad
\inf_{\widetilde u\in\mathcal A_C}\|u-\widetilde u\|=\|Q_Cu\|.
\]
Thus $r_{\rm rep}$ gives an irreducible error floor in this sector. The
closed-linear-subspace assumption is substantive: the calculation is not an
orthogonal decomposition of arbitrary nonlinear program classes. Nor does
representability certify that an admissible finite-budget search finds the
projected response.

\subsection{Proof of Proposition~\ref{prop:representation-necessity}}
\label{app:representation-necessity}
\begin{proof}
By the definition of the infimum, every $f\in\mathfrak R(C)$ satisfies
\[
D_\Xi(F,f\circ C)\ge\Omega_C(F).
\]
If $\Omega_C(F)>\varepsilon_\Xi$, each admissible relation law therefore
violates the declared tolerance. Under the fixed-boundary, fixed-evidence,
and exhaustive-relation-class assumptions, a successful internal successor
cannot remain in that representation class.
\end{proof}

This is a necessary condition on a successful repair, not an existence theorem
for a revised representation. A converse at
$\Omega_C(F)=\varepsilon_\Xi$ would require attainment; it is not asserted here.
A numerical search failure also does not certify the obstruction.

\subsection{Finite transitions and conditional global classification}
\label{app:global-mode-classification}
A finite transition $M^-\to M^+$ need not admit a tangent description.
Compare its representation classes on the same declared detailed-state
domain. If $C_{M^+}\cong C_{M^-}$, the abstraction is unchanged: the event
can be a fixed-representation relation change or a realization change.
Otherwise it contains a Representation revision. This comparison also
applies across engineering strata, so a discrete recompilation or
reparameterization is not automatically Representation learning.

Under Assumption A3 in Appendix~\ref{app:contracts}, these cases exhaust the
relevant persistent internal changes at the fixed boundary.
\begin{proof}[Argument under Assumption A3]
For each persistent transition, exactly one of
\[
[C_{M^+}]=[C_M]
\qquad\text{or}\qquad
[C_{M^+}]\ne[C_M]
\]
holds. In the first case, A3 places the represented learning content in
relation or realization changes within the current abstraction. In the
second, the representation class changes and relation changes may accompany
it. This gives the asserted conditional classification.
\end{proof}

The equality test is set-theoretic; the assertion that this description
captures every relevant persistent change is the completeness assumption
A3. The argument does not establish that assumption for arbitrary learners.
It also does not supply a global connection, a canonical finite
Transformation/Representation decomposition, or smooth coordinates across
singular strata. The exact-sequence theorem adds a local linear
factorization where its separate regularity assumptions hold. A pure
realization change or a no-op need not acquire any new capability, and this
classification alone does not establish preservation of old responsibilities.

\subsection{Boundary actions and persistent incorporation}
\label{app:mode-boundary-actions}
A one-time query, retrieval, sensor use, or intervention can resolve an
external information need without changing persistent learner state. An
already permitted call need not change the boundary itself. If new evidence
is incorporated persistently, its internal effect is classified by changes
to the relation law, the representation class, or both under A3.

If stateful interface machinery becomes part of the learner, include it in
the complete persistent state as required by Assumption A0 in
Appendix~\ref{app:contracts}, and restate the boundary and comparison domain.
Its modification is then classified by whether $C_M$ actually changes.
It is not automatically Representation learning: the modification may
preserve the abstraction and change a relation or realization instead.
Actual changes to information rights or the available action set require
that enlarged setting to be declared before applying a fixed-boundary
classification. External resolution is therefore distinct from the
conditional classification of persistent internal changes.

\section{Affine Stability--Plasticity Compatibility}
\label{app:compatibility}

All statements in this appendix use the finite-dimensional affine model of
Section~\ref{sec:dynamics-compatibility}: an edit $d\in\R^n$ satisfies
$Ad=0$, the updated residual is $r-Bd\in\R^m$, and the edit energy is
$d^\top Hd$. Here $A\in\R^{p\times n}$, $B\in\R^{m\times n}$, and
$H\in\R^{n\times n}$ is positive definite. These relations are exact in the
stated affine sector. When $A$ and $B$ are nonlinear learner Jacobians,
they describe first-order behavior only.

\subsection{The semantic-safe inverse}
Let $P_H$ be the matrix in Equation~\eqref{eq:safe-inverse}. Introduce
\[
\widetilde A=AH^{-1/2},\qquad
Q_A=I-\widetilde A^\top(\widetilde A\widetilde A^\top)^\dagger\widetilde A.
\]
Then $Q_A$ is the Euclidean orthogonal projector onto
$\ker\widetilde A$, and $P_H=H^{-1/2}Q_AH^{-1/2}$. In general, $P_H$
itself is not an idempotent Euclidean projector.

\begin{lemma}[Properties of the semantic-safe inverse]
\label{lem:safe-inverse}
The matrix $P_H$ satisfies
\begin{gather*}
P_H=P_H^\top\succeq0,\qquad AP_H=0,\qquad P_HA^\top=0,\\
\im P_H=\ker A,\qquad \ker P_H=\im A^\top,\\
P_HHd=d\quad\text{for every }d\in\ker A.
\end{gather*}
\end{lemma}
\begin{proof}
Because $Q_A$ is an orthogonal projector, congruence by $H^{-1/2}$ gives
symmetry and positive semidefiniteness of $P_H$. The identity
$\widetilde A Q_A=0$ gives $AP_H=0$; symmetry gives $P_HA^\top=0$.
Moreover,
\[
\im P_H=H^{-1/2}\ker(AH^{-1/2})\subseteq\ker A.
\]
For the reverse inclusion, let $z\in\ker A$ and put $u=H^{1/2}z$.
Then $AH^{-1/2}u=Az=0$, so $Q_Au=u$. Consequently,
\[
P_HHz=H^{-1/2}Q_AH^{1/2}z=z,
\]
which proves the range identity and the stated action on safe directions.
Finally, $P_Hx=0$ is equivalent to $Q_AH^{-1/2}x=0$, hence to
\[
H^{-1/2}x\in\im\widetilde A^\top
=\im(H^{-1/2}A^\top).
\]
Since $H^{-1/2}$ is invertible, this is equivalent to $x\in\im A^\top$.
\end{proof}

The residual-response operator $\Kplast=BP_HB^\top$ is therefore positive
semidefinite, and $r^\top\Kplast r=(B^\top r)^\top P_H(B^\top r)$ measures
the residual's coupling to safe edit directions.

\subsection{Proof of Theorem~\ref{thm:compatibility}}
\begin{proof}
Set $x=B^\top r$. If $P_Hx\ne0$, choose $d=P_Hx$. Then $Ad=0$ and
\[
\langle r,Bd\rangle=x^\top P_Hx>0,
\]
where strict positivity follows from positive semidefiniteness and
$P_Hx\ne0$. This proves that~\eqref{eq:compatible-covector} implies
\eqref{eq:compatible-direction}.

Conversely, if $P_Hx=0$, symmetry and $\im P_H=\ker A$ imply that $x$ is
orthogonal to $\ker A$. Hence $\langle r,Bd\rangle=x^\top d=0$ for every
$d\in\ker A$, excluding~\eqref{eq:compatible-direction}. Finally,
\[
r^\top\Kplast r=(B^\top r)^\top P_H(B^\top r),
\]
which is positive exactly when $P_HB^\top r\ne0$. Thus all three conditions
are equivalent.
Since $\Kplast\succeq0$, the same conditions are also equivalent to
$\Kplast r\ne0$.
\end{proof}

\subsection{The proximal affine step}
\label{app:proximal-step}
For $\eta>0$, the strictly convex quadratic objective
in~\eqref{eq:proximal-safe-step} has a unique minimizer. Its KKT conditions are
\[
\eta^{-1}Hd-B^\top(r-Bd)+A^\top\lambda=0,\qquad Ad=0.
\]
Multiplying the stationarity equation by $P_H$, and using
$P_HA^\top=0$ and $P_HHd=d$ on $\ker A$, gives
\[
d=\eta P_HB^\top(r-Bd).
\]
Writing $r^+=r-Bd$ gives $d=\eta P_HB^\top r^+$ and consequently
\[
(I+\eta\Kplast)r^+=r.
\]
Because $\Kplast\succeq0$, the matrix $I+\eta\Kplast$ is positive definite.
Substitution therefore yields~\eqref{eq:proximal-safe-solution}
and~\eqref{eq:proximal-safe-residual}. The identity $AP_H=0$ gives
$Ad^\star=0$.

\subsection{Weak, strict, and uniform residual contraction}
Write $\Kplast=U\Lambda U^\top$ with orthogonal $U$ and
$\Lambda=\operatorname{diag}(\lambda_i)$, $\lambda_i\ge0$, and let $r=Uc$.
The residual formula gives
\begin{equation}
\|r^+\|_2^2=\sum_i\frac{c_i^2}{(1+\eta\lambda_i)^2}
\le\sum_i c_i^2=\|r\|_2^2.
\label{eq:affine-spectral-contraction}
\end{equation}
The inequality is strict exactly when $r$ has a nonzero component in a
positive-eigenvalue eigenspace, equivalently $\Kplast r\ne0$.
If $r\in\ker\Kplast$, then $r^+=r$. Thus positive semidefiniteness gives
weak contraction, not strict improvement for every residual.

If the stronger bound $\Kplast\succeq\kappa I$ holds on the entire
registered residual space $\R^m$, with $\kappa>0$, the same calculation gives
\[
\|r^+\|_2\le(1+\eta\kappa)^{-1}\|r\|_2.
\]
This uniform statement uses a lower bound on every registered direction.
A quadratic-form lower bound only on an arbitrary proper subspace is not
used here as a certificate for the full resolvent's contraction on that
subspace.

\subsection{Minimum energy for exact affine closure}
\label{app:minimum-edit-energy}
\begin{proposition}[Minimum semantic-safe edit energy]
\label{prop:minimum-edit-energy}
The constraints $Ad=0$ and $Bd=r$ have a solution if and only if
$r\in\im\Kplast$. When they are feasible,
\begin{equation}
\min_{Ad=0,\,Bd=r}d^\top Hd=r^\top\Kplast^\dagger r.
\label{eq:minimum-edit-energy}
\end{equation}
\end{proposition}
\begin{proof}
Let the columns of $Z$ form a basis of $\ker A$. Every safe edit is
$d=Zz$. Define $H_T=Z^\top HZ\succ0$ and $G=BZ$. The problem becomes
\[
\min_z z^\top H_Tz\quad\text{subject to }Gz=r.
\]
Its feasible residual space is $\im G$, and the standard minimum-energy
formula gives the value $r^\top(GH_T^{-1}G^\top)^\dagger r$ there.
The constrained inverse identity
\[
Z(Z^\top HZ)^{-1}Z^\top=P_H
\]
implies $GH_T^{-1}G^\top=BP_HB^\top=\Kplast$. The feasible residual space
is consequently $\im\Kplast$, and the minimum energy is
Equation~\eqref{eq:minimum-edit-energy}.
\end{proof}
The basis reduction uses the usual zero-dimensional convention if
$\ker A=\{0\}$; only the zero edit is then available.

\subsection{A geometric load for the stated edit budget}
For a positive energy radius $R_0>0$ and budget $d^\top Hd\le R_0^2$, define
\begin{equation}
q_{\rm geo}(r)=
\begin{cases}
\sqrt{r^\top\Kplast^\dagger r}/R_0,&r\in\im\Kplast,\\
+\infty,&r\notin\im\Kplast.
\end{cases}
\label{eq:geometric-closure-load}
\end{equation}
Proposition~\ref{prop:minimum-edit-energy} makes $q_{\rm geo}(r)\le1$
equivalent to an exact affine semantic-safe closure within this energy
budget. The range condition cannot be dropped: the pseudoinverse quadratic
form alone does not certify feasibility.

This is an energy-normalized local load. Identifying it with the general
operational closure load requires a matching edit family, resource scale,
and full-cost convention. Calling an infeasible affine response a
Transformation obstruction additionally requires the coordinates to be
restricted to fixed-representation edits. Even then it does not establish
global Representation necessity or exclude a nonlocal same-representation
path.

\subsection{A fixed linear continuous-time flow}
\label{app:fixed-linear-flow}
Keep $A$, $B$, and $H$ fixed, start with $Ad(0)=0$, and normalize time so that the safe edit
velocity is $\dot d=P_HB^\top r$. With $r=r_0-Bd$, the induced residual
equation and its energy derivative are
\begin{equation}
\dot r=-\Kplast r,\qquad
\frac{d}{dt}\frac12\|r\|_2^2=-r^\top\Kplast r\le0.
\label{eq:fixed-linear-residual-flow}
\end{equation}
If $\Kplast\succeq\kappa I$ on the entire registered residual space,
the displayed inequality gives the stated bound
\[
\|r(t)\|_2\le e^{-\kappa t}\|r(0)\|_2.
\]
These statements concern fixed affine operators. They do not provide a
state-dependent nonlinear flow theorem, finite nonlinear endpoint
preservation, or a derivation of the full discrete GRC procedure as a
continuous-time limit. A rate for an arbitrary proper residual subspace
is not asserted without further subspace conditions.

\subsection{Reconstructive semantic stability dominance}
\label{app:reconstructive-dominance}
Let $T_{\rm exact}\subseteq\R^n$ be the linear safe space obtained by
preserving a historical realization exactly, and let $T_{\rm rec}$ be the
linear safe space obtained by preserving the registered future-operational
semantics while permitting a different realization. The comparison assumes
\begin{equation}
 T_{\rm exact}\subseteq T_{\rm rec}.
 \label{eq:reconstructive-nested-spaces}
\end{equation}
This is a condition on the declared preservation rules: exact preservation
must impose the semantic requirements together with additional realization
constraints. It need not hold for arbitrary pairs of constraint systems.
Throughout the comparison, the positive-definite metric $H$ and the
residual-response map $B$ are fixed.

For a linear subspace $T$, define its safe inverse by
\begin{equation}
 P_T^H=H^{-1/2}\Pi_{H^{1/2}T}H^{-1/2},
 \label{eq:subspace-safe-inverse}
\end{equation}
where $\Pi_{H^{1/2}T}$ is the Euclidean orthogonal projector onto the
whitened subspace $H^{1/2}T$. For $T=\ker A$, this is the operator $P_H$
of Lemma~\ref{lem:safe-inverse}. Write
$P_j^H=P_{T_j}^H$ and
$K_{\rm plast}^j=BP_j^HB^\top$ for
$j\in\{\mathrm{exact},\mathrm{rec}\}$.

\begin{theorem}[Reconstructive semantic stability dominance]
\label{thm:reconstructive-semantic-dominance}
Under~\eqref{eq:reconstructive-nested-spaces},
\begin{equation}
 \begin{aligned}
 P_{\rm rec}^H-P_{\rm exact}^H&\succeq0,\\
 K_{\rm plast}^{\rm rec}-K_{\rm plast}^{\rm exact}
 &=B(P_{\rm rec}^H-P_{\rm exact}^H)B^\top\succeq0.
 \end{aligned}
 \label{eq:reconstructive-psd-proof}
\end{equation}
In particular, this establishes the response-operator ordering
in~\eqref{eq:reconstructive-dominance}.
\end{theorem}

\begin{proof}
Since $H^{1/2}$ is invertible, the whitened spaces satisfy
$H^{1/2}T_{\rm exact}\subseteq H^{1/2}T_{\rm rec}$.
For nested Euclidean subspaces $S_1\subseteq S_2$, the difference
$\Pi_{S_2}-\Pi_{S_1}$ is the orthogonal projector onto
$S_2\cap S_1^\perp$ and is therefore positive semidefinite. Consequently,
\[
 \Pi_{H^{1/2}T_{\rm rec}}-\Pi_{H^{1/2}T_{\rm exact}}\succeq0.
\]
Congruence by $H^{-1/2}$ gives the first ordering
in~\eqref{eq:reconstructive-psd-proof}; congruence by $B$ gives the second.
\end{proof}

\subsection{When the safe response improves strictly}
The ordering is weak in general. A strict gain exists for at least one
residual direction if and only if
\begin{equation}
 B(P_{\rm rec}^H-P_{\rm exact}^H)B^\top\ne0.
 \label{eq:reconstructive-nonzero-gain}
\end{equation}
For a specified residual $r$, the response quadratic form improves strictly
exactly when
\begin{equation}
 r^\top B(P_{\rm rec}^H-P_{\rm exact}^H)B^\top r>0.
 \label{eq:reconstructive-directional-gain}
\end{equation}
Thus a strict enlargement of the safe space alone is insufficient: the
newly available directions must couple to $B$ and, for a given residual,
to that residual direction. These are statements about the local response
quadratic form, not a general ordering of finite-step residual norms.

\subsection{Closure-load comparison on a common nondegenerate support}
\label{app:reconstructive-load}
Let $\mathcal Y\subseteq\R^m$ be a common invariant residual subspace of
$K_{\rm plast}^{\rm exact}$ and $K_{\rm plast}^{\rm rec}$ on which the
former is positive definite. In particular, the comparison applies on a
shared image space when both operators have that image. Invariance makes
the restrictions operators on $\mathcal Y$; an arbitrary compression to a
subspace is not being identified with the original closure problem.
Set
\[
 K_j^{\mathcal Y}=K_{\rm plast}^j\big|_{\mathcal Y},
 \qquad j\in\{\mathrm{exact},\mathrm{rec}\}.
\]
The PSD comparison gives
$K_{\rm rec}^{\mathcal Y}\succeq K_{\rm exact}^{\mathcal Y}\succ0$.
Positive-definite inversion reverses L\"owner order, hence
\begin{equation}
 (K_{\rm rec}^{\mathcal Y})^{-1}
 \preceq (K_{\rm exact}^{\mathcal Y})^{-1}.
 \label{eq:reconstructive-support-inverse}
\end{equation}
With the same edit-budget radius $R_0>0$, the geometric load
of~\eqref{eq:geometric-closure-load} therefore satisfies, for $r\in\mathcal Y$,
\begin{equation}
 \begin{aligned}
 q_{\rm geo}^{\rm rec}(r)
 &=\frac{\sqrt{r^\top(K_{\rm rec}^{\mathcal Y})^{-1}r}}{R_0}\\
 &\le
 \frac{\sqrt{r^\top(K_{\rm exact}^{\mathcal Y})^{-1}r}}{R_0}
 =q_{\rm geo}^{\rm exact}(r).
 \end{aligned}
 \label{eq:reconstructive-load-dominance}
\end{equation}
The inequality is strict whenever
\[
 r^\top\bigl[(K_{\rm exact}^{\mathcal Y})^{-1}
 -(K_{\rm rec}^{\mathcal Y})^{-1}\bigr]r>0.
\]
Every inverse here is taken on $\mathcal Y$. No ambient pseudoinverse-order
claim is made outside a common nondegenerate support.

\subsection{Interpretation and scope}
Under the stated inclusion, imposing additional realization constraints
that the registered semantics do not require cannot increase safe
plasticity in the response-operator sense above. Forgetting violates a
registered preservation requirement; rigidity preserves more than that
contract requires. The distinction does not make all historical constraints
unnecessary or establish that a reconstructive realization is inexpensive
to discover or execute.

The comparison holds $H$, $B$, and the residual contract fixed; the load
comparison also holds the edit budget fixed. It does not compare different
metrics, response maps, total search or reconstruction costs, or arbitrary
nonlinear finite updates. In particular, PSD ordering alone is not used to
claim that every residual has a smaller Euclidean norm after a finite
proximal step.

\section{Representation Obstruction, Closure Load, and Structural Carry}
\label{app:grade}

\subsection{Solver Failure and Representation Obstruction}
\label{app:grade-obstruction}
A failed solve does not identify its cause. A solver may miss an available
relation, explore too narrow a family, or lack evidence available only
through another declared interface. These possibilities differ from the
absence of any admissible relation law on the current representation.

For consistency with Section~\ref{sec:modes}, write $\mathfrak R(C)$ for the
exhaustive admissible relation family on representation $C$. Let
$\mathcal A_{\mathrm{search}}(C)\subseteq\mathfrak R(C)$ be the family
actually explored, and let $F$ denote the required future response. Its
search-restricted obstruction is
\begin{equation}
\Omega_C^{\mathrm{search}}(F)
=\inf_{f\in\mathcal A_{\mathrm{search}}(C)}D_\Xi(F,f\circ C).
\label{eq:grade-search-obstruction}
\end{equation}
A certified value above $\varepsilon_\Xi$ rules out that searched family,
not the entire representation. The loss of one returned candidate does not
by itself certify this infimum either.

The exhaustive quantity is $\Omega_C(F)$ from
Equation~\eqref{eq:representation-obstruction}. Under the fixed boundary,
evidence, tolerance, and exhaustive-family assumptions,
$\Omega_C(F)>\varepsilon_\Xi$ excludes relation-only closure by
Proposition~\ref{prop:representation-necessity}. Its existing proof is in
Appendix~\ref{app:representation-necessity}; no separate proposition is
needed here. Diagnosing this inequality from a computational failure
requires ruling out incomplete search at the level of the claim.

The inequality $\Omega_C(F)\le\varepsilon_\Xi$ does not automatically
certify an attained feasible relation at the tolerance boundary. It also
does not guarantee that a chosen solver finds a feasible relation, or
exclude a revised representation with lower cost or better future capacity.

\subsection{A Residual-Space Correction Body}
\label{app:grade-body}
Closure load compares a residual with attainable corrections in a common
residual space. Let $\mathcal F^\Xi_{[C_M]}$ be the lawful realization fibre
that preserves the representation class $[C_M]$. In particular, members
preserve $\varpi_\Xi(M)$, not every acquired future response. Denote by
$\Delta r(M,N)$ the correction induced by an admissible successor $N$ in
that common residual space. The unit-cost body is
\begin{equation}
\mathcal C_T^\Xi(M)=
\left\{\Delta r(M,N):
\begin{array}{l}
 N\in\mathcal F^\Xi_{[C_M]},\\
 C_{\mathrm{full}}(N\mid M)\le1
\end{array}\right\}.
\label{eq:grade-correction-body}
\end{equation}
The unit is a declared cost normalization. Each member must be realized by
one lawful successor satisfying the protected duties together. Separately
feasible corrections for different duties need not certify one joint edit.
The body is not a set of machine states and need not be convex, connected,
or star-shaped.

\subsection{Calibrated Load and Its Boundary}
\label{app:grade-load}
Declare a family $D_\lambda\mathcal C_T^\Xi(M)$ of reachable corrections,
nested in $\lambda>0$. Interpreting $\lambda=1$ as the original unit budget
requires $D_1\mathcal C_T^\Xi(M)=\mathcal C_T^\Xi(M)$. The load is
Equation~\eqref{eq:closure-load}, with the convention
$\inf\varnothing=+\infty$.
This definition does not make the reachable sets identifiable or
computable. The scalar-calibration requirements in
Appendix~\ref{app:contracts} are additional obligations.

Where the declared scaling is calibrated to actual feasible edits, a load
below one indicates a correction within the unit budget, and a load above
one rules out a correction in that declared unit-cost family. This is a
budget-relative statement. ``Below one'' is not a general assertion of
topological interior. At $q_T=1$, the value of an infimum alone does not
certify a feasible boundary correction: the correspondence between a load
threshold and membership must include the necessary attainment or set
regularity conditions. No such conditions are presumed for an arbitrary
nonlinear family.

\paragraph{Radial gauge sector.}
For standard radial dilation $D_\lambda\mathcal C=\lambda\mathcal C$, an
absorbing body star-shaped about the origin gives the corresponding
Minkowski gauge. Convex gauge properties are invoked only in a convex
sector, with the origin, scaling, and any boundary regularity required by
the particular property specified. This specialization is not a claim that
all reachable correction sets have gauge geometry.

\paragraph{Several protected constraints.}
If a common calibrated radial description genuinely has joint correction
body $\mathcal C=\bigcap_{i=1}^{m}\mathcal C_i$, the proposed scalar
consolidation is
\begin{equation}
q(r)=\max_{1\le i\le m}q_i(r).
\label{eq:grade-calibrated-max}
\end{equation}
Using this identity requires compatible gauge assumptions and a dilation
that represents the same cost scale for every component. In particular,
intersecting componentwise reachable images must not be substituted for
certifying a single edit that meets all constraints. The identity is not
asserted for arbitrary response sets or differently calibrated diagnostics;
those cases retain a vector or certificate description. A complete P3
argument must establish the joint-set and gauge conditions before applying
the scalar formula to a learner family.

\subsection{Bounds and Certificate-Valued Diagnosis}
\label{app:grade-certificates}
Disconnected, non-star-shaped, or partially identified correction sets need
not admit a useful scalar summary. An interval
$[\underline q,\overline q]$ can record certified load bounds when a common
scalar comparison is justified. Otherwise retain
\begin{equation}
\operatorname{Cert}_T(r\mid M)\in
\{\mathrm{INSIDE},\mathrm{CROSSES},\mathrm{OUTSIDE},\mathrm{UNIDENTIFIED}\}.
\label{eq:grade-certificates}
\end{equation}
\begin{description}
\item[INSIDE.] A lawful same-representation closure within the declared
budget is certified.
\item[OUTSIDE.] The absence of such a closure in the declared family is
certified, with the comparison scope fixed.
\item[CROSSES.] Available certified bounds straddle the threshold and do not
settle feasibility, including the applicable boundary qualification.
\item[UNIDENTIFIED.] Calibration or evidence is insufficient for a valid
decision.
\end{description}
An implementation using inner and outer approximations must state which
reachable set they bound and certify the corresponding inclusions. An
inner witness must realize the joint obligations; an outer exclusion must
cover the full declared family. No construction of such bounds, or
general certificate-completeness theorem, is supplied by the load
definition. Unknown calibration is not evidence for structural growth.

\subsection{When Structural Carry Is Justified}
\label{app:grade-carry}
The implication $q_T>1\Rightarrow\text{Representation Learning}$ is invalid.
An overflow can reflect an exhausted budget, a changed background,
unavailable evidence, or incomplete search. The diagnostic ordering
\[
\begin{gathered}
\text{same-grade saturation}\ \longrightarrow\ \text{transport}
\ \longrightarrow\ \text{boundary/evidence}\\
\longrightarrow\ \text{search/solver}
\ \longrightarrow\ \text{representation obstruction}
\end{gathered}
\]
expresses logical distinctions, not a mandatory execution schedule.

For a persistent, future-relevant defect, a structural Carry requires the
following claims to be established under the declared information and cost
contract:
\begin{enumerate}
\item Current same-representation closure is infeasible at the stated scope,
and lawful transport does not remove the defect.
\item Missing admissible evidence, an inadequate comparison budget, and
search or solver incompleteness do not explain the asserted representation
obstruction.
\item The exhaustive obstruction satisfies
$\Omega_C(F)>\varepsilon_\Xi$, and a lawful revised representation exists.
\end{enumerate}
These conditions distinguish the necessity of leaving the present class
from the existence of a successful successor. The notation
\[
M^+=\operatorname{LeastLawfulCarry}_\Xi(M,r)
\]
specifies a least-lawful structural selection; it does not construct the
successor, prove an attained minimum, or provide a solver. Carry may create,
split, merge, or revise a state, owner, memory variable, or program
component. It is not synonymous with adding parameters.

\subsection{Filtration, Closure Valuation, and Learning Modes}
\label{app:grade-filtration}
Let $F_0(M)\subseteq F_1(M)\subseteq\cdots$ be lawful residual-response
families on a common domain, where $F_s(M)$ contains distinctions closable
at grade $s$. The costs, evidence rights, protected duties, and tolerances
must be comparable. A calibrated description may identify
\begin{equation}
F_s(M)=\{r:q_{\Xi,s}^{M}(r)\le1\},
\label{eq:grade-threshold-family}
\end{equation}
only after the threshold-to-feasibility and boundary conditions are
established. Under that identification, define the closure valuation
\begin{equation}
 \nu_\Xi(r\mid M)=\inf\{s:q_{\Xi,s}^M(r)\le1\},
 \label{eq:closure-grade}
\end{equation}
using $+\infty$ when no grade qualifies.

The load $q_s$ asks about closure within a grade; $\nu_\Xi$ asks which grade
can first qualify. Transformation/Representation classification asks
instead whether the representation structure changes. Same-grade closure
is pure Transformation only when that grade's edit family fixes the
representation schema. Likewise, a larger grade must not be equated with
representation revision merely because it permits a larger budget or
search space. The implication from $\nu_\Xi(r)>s$ to required structural
revision additionally needs the structural meaning of the grades and
the exclusions in the Carry diagnosis.

\subsection{Associated Grades Require Additional Structure}
\label{app:grade-associated}
For an actual filtration of vector spaces or modules, with the appropriate
subspace or submodule inclusions, one may form
\begin{equation}
V_s=F_s/F_{s-1}.
\label{eq:grade-associated}
\end{equation}
A class of $r\in F_s\setminus F_{s-1}$ then records a distinction not
absorbed at the preceding grade. Cost-bounded reachable sets are not
automatically such spaces. This notation therefore does not define a
quotient for general nonlinear closure families.

An alternative notation $F_{s+1}/\!\sim_s$ also requires a genuine
equivalence relation on the stated domain. Saying that the difference
between two residuals is absorbable at grade $s$ does not by itself
establish reflexivity, symmetry, or transitivity in an arbitrary reachable
set; even the meaning of a difference must be declared. No nonlinear
equivalence theorem or general associated-graded algebra is established
here.

\paragraph{Status of P3.}
The correction bodies, loads, diagnostic certificates, and filtration
provide a conditional organization of closure. A complete P3 result still
requires the claimed calibration, attainability, joint-feasibility, and
structural-grade assumptions, together with any bound or Carry
construction. The representation-necessity implication remains P2 and its
existing proof. An OUTSIDE certificate alone is not a CARRY certificate.

\section{Dynamic Sufficiency, Growth, and Compilation}
\label{app:growth}

\subsection{Update-Sufficiency Factorization}
\label{app:update-factorization}
Let $\widetilde C_t:\Hist_t\to\widetilde Z_t$ be a candidate compressed representation, $U_e:\Hist_t\to\Hist_{t+1}$ a registered detailed update, and $C_{t+1}:\Hist_{t+1}\to Z_{t+1}$ the required post-update semantic encoding. We take $\widetilde Z_t=\widetilde C_t(\Hist_t)$, so the claimed updater is defined on attained compressed states.

\begin{proposition}[Update-sufficiency factorization criterion]
\label{prop:update-factorization}
There is a map $\bar U_e:\widetilde Z_t\to Z_{t+1}$ satisfying
\begin{equation}
 C_{t+1}\circ U_e=\bar U_e\circ\widetilde C_t
 \label{eq:update-factorization-exact}
\end{equation}
if and only if the required continuation is constant on the fibres of $\widetilde C_t$:
\begin{equation}
 \widetilde C_t(x_1)=\widetilde C_t(x_2)
 \ \Longrightarrow\ C_{t+1}U_e(x_1)=C_{t+1}U_e(x_2)
 \label{eq:update-fibre-criterion}
\end{equation}
for all $x_1,x_2\in\Hist_t$.
\end{proposition}
\begin{proof}
If the updater exists, equal compressed inputs give
\[
 C_{t+1}U_e(x_1)=\bar U_e\widetilde C_t(x_1)
 =\bar U_e\widetilde C_t(x_2)=C_{t+1}U_e(x_2).
\]
Conversely, for $k=\widetilde C_t(x)$ define $\bar U_e(k)=C_{t+1}U_e(x)$. Fibre constancy makes this independent of the choice of representative $x$. The map is therefore well defined on the stated image and satisfies Equation~\eqref{eq:update-factorization-exact}.
\end{proof}
For contract-relative identity, the equality is taken in the appropriate semantic quotient. This mathematical factorization criterion concerns retained information; admissible implementation and resource requirements are separate obligations. It does not establish a quantitative approximate factorization result.

\subsection{Current Prediction Need Not Determine a Future Update}
\label{app:update-counterexample}
The supplied example has two detailed histories $\Hist_t=\{x_1,x_2\}$ requiring the same current output. The encoding $\widetilde C(x_1)=\widetilde C(x_2)=k$ is sufficient for that output. After a registered feedback event, let
\[
 U_e(x_1)=a,\qquad U_e(x_2)=b,\qquad
 C^+(a)=k_a,\quad C^+(b)=k_b,\quad k_a\ne k_b.
\]
An updater based only on the shared compressed state would have to satisfy both $\bar U_e(k)=k_a$ and $\bar U_e(k)=k_b$, a contradiction. Current-prediction sufficiency therefore does not imply update sufficiency.

The example uses a candidate prediction compression, not an already sufficient future-operational quotient. If the contract registers the differing update behavior, its full quotient cannot silently merge the two histories. Updating contexts must therefore contribute to the declared future equivalence.

\subsection{A Family of Lawful Continuations}
Update sufficiency requires a lawful induced updater for every registered event, with the appropriate source and target encodings as in Equation~\eqref{eq:update-sufficiency}. Dynamic sufficiency extends this requirement to all registered evolution classes, including evidence updates, background transport, and internal control transitions. The relevant commuting diagrams must each be declared; a single-event factorization does not supply arbitrary background transports or their composition. The interpretation that knowledge retains both what is known and how it should change describes this information requirement, rather than adding another mathematical assumption.

\subsection{Future-closure Viability and Lawful Successor Paths}
\label{future-viability}
Let $\mathcal Y_\Xi^{\rm fut}$ be the registered class of future residual
demands. A lawful successor closure path for $r$ is a finite sequence
$\gamma=(M_0,\ldots,M_L)$, starting at $M_0=M$, whose persistent
transitions satisfy the declared semantic, dynamic-sufficiency, resource,
and safety requirements, and whose final state closes $r$ to the contract
tolerance. Thus the path-based definition of Section~\ref{sec:growth} is
\begin{equation}
 \mathcal V_\Xi(M)=
 \bigl\{r\in\mathcal Y_\Xi^{\rm fut}:
       \Gamma_\Xi(M,r)\ne\varnothing\bigr\}.
 \label{future-viability-paths}
\end{equation}
Future viability on a declared target class $\mathcal Y_0$ requires
$\mathcal Y_0\subseteq\mathcal V_\Xi(M)$. It is an existence condition
for lawful continuation paths, with no claim about unregistered futures.

\subsection{Sequential Safe Inverse}
Work in the finite-dimensional affine geometry of
Section~\ref{sec:dynamics-compatibility}. Let $P=P^\top\succeq0$ denote
the current semantic-safe inverse for the declared edit metric and previous
responsibilities. Suppose a newly acquired relation is registered as an
additional hard local responsibility $Dd=0$. The remaining safe inverse is
\begin{equation}
 P_{\rm after}=P-PD^\top(DPD^\top)^\dagger DP.
 \label{eq:growth-local-geometry}
\end{equation}
The following supplied linear-algebra result permits rank-deficient $P$
and $D$; the dagger denotes the Moore--Penrose pseudoinverse.

\begin{lemma}[Sequential safe inverse]
\label{future-sequential-safe-inverse}
Let $P$ be a finite-dimensional symmetric positive-semidefinite matrix,
let $D$ be a compatible real matrix, and set $G=DP^{1/2}$. Then
\begin{equation}
 P_{\rm after}
 =P^{1/2}\bigl[I-G^\top(GG^\top)^\dagger G\bigr]P^{1/2},
 \qquad 0\preceq P_{\rm after}\preceq P,
 \label{future-safe-factorization}
\end{equation}
and
\begin{equation}
 \im P_{\rm after}=\im P\cap\ker D.
 \label{future-safe-range}
\end{equation}
\end{lemma}

\begin{proof}
The identities $PD^\top=P^{1/2}G^\top$,
$DPD^\top=GG^\top$, and $DP=GP^{1/2}$ give the factorization by
substitution into Equation~\eqref{eq:growth-local-geometry}. The bracketed
matrix is the Euclidean orthogonal projector onto $\ker G$. It is
positive semidefinite and bounded above by $I$, so congruence by $P^{1/2}$
gives $0\preceq P_{\rm after}\preceq P$. The range of $P_{\rm after}$ consists precisely
of the vectors in the previous safe space whose image under $D$ vanishes,
giving Equation~\eqref{future-safe-range}.
\end{proof}

Registering an additional responsibility can therefore consume local safe
directions. Successful current closure alone does not establish that a
useful response remains available for the next registered demand.

\subsection{A Certificate for Registered Local Future Probes}
Let $B_f$ be the response Jacobian of a registered local probe family.
Its post-update response operator is
\begin{equation}
 K_f^{\rm after}=B_fP_{\rm after}B_f^\top.
 \label{future-probe-operator}
\end{equation}
The probe family and its construction must respect the declared information
rights. The notation $B_f$ neither assumes nor authorizes access to unavailable
future-task identities or target labels; it does not encode all unknown future tasks.

\begin{proposition}[Registered local future-plasticity certificate]
\label{future-registered-probe}
Suppose the declared probe operator satisfies
\begin{equation}
 B_fP_{\rm after}B_f^\top\succeq\kappa_f I,
 \qquad \kappa_f>0.
 \label{eq:growth-local-probe}
\end{equation}
Then every nonzero $r_f$ in the registered probe output space has
$P_{\rm after}B_f^\top r_f\ne0$. In the corresponding frozen affine
proximal model, for $\eta>0$, its residual update satisfies
\begin{equation}
 r_f^+=(I+\eta K_f^{\rm after})^{-1}r_f,
 \qquad
 \|r_f^+\|\le\frac{1}{1+\eta\kappa_f}\|r_f\|.
 \label{future-probe-contraction}
\end{equation}
\end{proposition}

\begin{proof}
For $r_f\ne0$,
\[
 r_f^\top B_fP_{\rm after}B_f^\top r_f
 \ge\kappa_f\|r_f\|^2>0.
\]
The compatibility criterion of Theorem~\ref{thm:compatibility}, applied
with the additional responsibility, gives a nonzero semantic-safe response.
The corresponding proximal residual operator is
$(I+\eta K_f^{\rm after})^{-1}$, as in
Equation~\eqref{eq:proximal-safe-residual}. Every eigenvalue of
$K_f^{\rm after}$ is at least $\kappa_f$, so
\[
 \bigl\|(I+\eta K_f^{\rm after})^{-1}\bigr\|_2
 \le(1+\eta\kappa_f)^{-1}.
\]
The residual contraction bound follows.
\end{proof}

The certificate concerns the declared linear probe directions. For a
nonlinear learner, this calculation is local and does not certify its
finite endpoint or a sequence of later updates. It also does not establish
$\mathcal V_\Xi(M)=\mathcal Y_\Xi^{\rm fut}$: closure paths can require
nonlocal changes, evidence, or representation revision absent from $B_f$.
Conversely, a zero local response need not imply global non-viability when
a lawful representation revision remains available.

\subsection{A Minimal Collapse Example and Continual Admissibility}
\begin{example}[Loss of a registered future response]
\label{future-collapse-example}
Take a two-dimensional current safe edit space with $P=I_2$, and register
$D=\begin{bmatrix}1&0\end{bmatrix}$ as a new responsibility. Then
$DPD^\top=1$ and
\begin{equation}
 P_{\rm after}
 =I_2-\begin{bmatrix}1\\0\end{bmatrix}
       \begin{bmatrix}1&0\end{bmatrix}
 =\begin{bmatrix}0&0\\0&1\end{bmatrix}.
 \label{future-collapse-safe-inverse}
\end{equation}
For the registered future probe
$B_f=\begin{bmatrix}1&0\end{bmatrix}$, this gives
\begin{equation}
 B_fP_{\rm after}B_f^\top=0.
 \label{future-collapse-response}
\end{equation}
The current responsibility can thus be installed while leaving no
same-representation first-order response for that future probe, even
though the previous safe requirements remain in force.
\end{example}

This example distinguishes retaining old responsibilities and acquiring
the present one from retaining a particular future local response. It
does not exclude a nonlocal or representation-changing repair.
At the general level, continual admissibility requires the declared
future-closure condition $\mathcal Y_0\subseteq\mathcal V_\Xi(M^+)$,
along with semantic preservation, current closure, and correct updating.
Where local probes are registered, Equation~\eqref{eq:growth-local-probe}
provides an additional finite certificate. It does not replace the
multi-step viability requirement.

\subsection{Path Costs and Capability Comparison}
Fix a registered future residual class $\mathcal Y_\Xi^{\rm fut}$.
The learner state $M$ includes the persistent variables relevant to future
closure; it is distinct from a detailed interaction state $x$ or the
represented state $C_M(x)$. Previously acquired duties are registered under
$\Xi_{\rm old}$ and retained within the current contract.
Every element of $\Gamma_\Xi(M,r)$ must be a finite lawful closure path under
the same registered information rights, semantic responsibilities, resource
constraints, and closure tolerance. We use finite nonnegative values of
$C_{\rm full}$ on these paths. An empty path set has optimal cost $+\infty$.
Thus the optimal complete future-closure cost and its finite domain are
\[
 c^\star_\Xi(r\mid M)
 =\inf_{\gamma\in\Gamma_\Xi(M,r)}C_{\rm full}(\gamma),
 \qquad
 \mathcal V_\Xi(M)=\{r:c^\star_\Xi(r\mid M)<+\infty\}.
\]
Unknown costs do not establish feasibility or a dominance certificate.
The epigraph in Equation~\eqref{eq:growth-frontier} uses finite cost
coordinates, so demands with no lawful path contribute no points.
It compares infimal costs, not necessarily attained budgets; a statement
that a path exists at cost exactly $c^\star_\Xi$ would additionally require
attainment.

Write $M'\succeq_\Xi^{\rm sem}M$ when every previously registered
future-operational capability remains lawfully reconstructible from $M'$.
Assume lawful self-reconstruction and composability of these reconstruction
relations. For a common contract, define
\begin{equation}
 M'\succeq_\Xi M
 \quad\Longleftrightarrow\quad
 M'\succeq_\Xi^{\rm sem}M
 \ \text{and}\ 
 \mathscr F_\Xi(M')\supseteq\mathscr F_\Xi(M).
 \label{eq:growth-preorder}
\end{equation}
The frontier clause preserves both the viable domain and its optimal
complete costs. In particular, it requires
\[
 c^\star_\Xi(r\mid M')\le c^\star_\Xi(r\mid M)
 \quad\text{for every }r\in\mathcal V_\Xi(M),
\]
in addition to preservation of registered semantics. Unknown costs cannot
be treated as zero or as a certificate of this universal comparison.

\begin{proposition}[Closure-capability preorder]
\label{prop:growth-preorder}
Under the stated reconstruction assumptions, $\succeq_\Xi$ is reflexive
and transitive.
\end{proposition}
\begin{proof}
Self-reconstruction gives $M\succeq_\Xi^{\rm sem}M$, and the frontier
contains itself. Thus $M\succeq_\Xi M$.
If $M_2\succeq_\Xi M_1$ and $M_1\succeq_\Xi M_0$, composability gives
$M_2\succeq_\Xi^{\rm sem}M_0$, while
\[
\mathscr F_\Xi(M_2)\supseteq\mathscr F_\Xi(M_1)
\supseteq\mathscr F_\Xi(M_0).
\]
Together these imply $M_2\succeq_\Xi M_0$.
\end{proof}

\paragraph{Strict growth.}
Strict growth, denoted $M'\succ_\Xi M$, requires
the weak comparison together with
\[
 \mathcal V_\Xi(M')\supsetneq\mathcal V_\Xi(M)
 \quad\text{or}\quad
 c^\star_\Xi(r\mid M')<c^\star_\Xi(r\mid M)
 \text{ for some }r\in\mathcal V_\Xi(M).
\]
A decrease of the present residual alone does not establish either strict
condition. These comparisons are explicitly relative to the declared future
class. They make no guarantee about unregistered demands.
Parameter count, current accuracy, and the number of stored tasks are not
substitutes for these domain or optimal-cost conditions.

\subsection{Transport and Proof of T3}
\label{app:growth-proof}
If the background changes, we compare a successor with
$\widetilde M_t=\mathsf T_{t\to t+1}M_t$, rather than directly with $M_t$.
Transport must identify the old responsibilities and capability comparisons
in the new contract. It is the declared lawful no-learning baseline;
environmental transport is not counted as learned growth.

Speculative candidates can fail an audit, regress, or be discarded before
commit. Denote the candidate proposed for persistence by $M_t^\star$.
The fail-closed rule is
\begin{equation}
 M_{t+1}=
 \begin{cases}
 M_t^\star,
 &\begin{array}{l}
   \operatorname{Audit}_{\Xi_{t+1}}(M_t^\star)=\mathrm{PASS},\\[-1mm]
   M_t^\star\succeq_{\Xi_{t+1}}\widetilde M_t,
  \end{array}\\[1mm]
 \widetilde M_t,&\text{otherwise}.
 \end{cases}
 \label{eq:growth-commit}
\end{equation}
Audit soundness means that acceptance certifies current closure, preservation
of registered semantics, dynamic sufficiency, declared future viability,
and the resource and safety contract. The dominance gate is a further
assumption, not a consequence of current closure.

\begin{proof}[Proof of Theorem~\ref{thm:growth}]
There are two cases. If the candidate fails the audit or the dominance
condition, Equation~\eqref{eq:growth-commit} sets
$M_{t+1}=\widetilde M_t$. By reflexivity,
$M_{t+1}\succeq_{\Xi_{t+1}}\widetilde M_t$.
If the candidate is accepted, the commit condition requires
$M_t^\star\succeq_{\Xi_{t+1}}\widetilde M_t$, and the rule sets
$M_{t+1}=M_t^\star$. Thus the weak comparison holds in both cases.
If an accepted candidate additionally expands the lawful closure domain or
strictly reduces the optimal cost of a previously closable demand, the
definition of $\succ_{\Xi_{t+1}}$ gives the strict comparison.
\end{proof}

This proof concerns certified persistent states. Trial candidates can fail
the audit and can be discarded. The fallback branch does not prove that
the current residual is closed, and transport-relative non-regression does
not establish non-regression across arbitrarily changing environments.
The theorem also provides no method for checking optimal-cost dominance
over an infinite future class.

\paragraph{Ideal comparison and executable certification.}
The frontier comparison quantifies over every registered future residual
and over optimal lawful path costs. It is an ideal semantic condition,
not automatically a terminating audit procedure. A practical commit
requires certificates sound for that comparison, or a declared finite
sector where the required decisions and costs can actually be evaluated.
Finite task or candidate counts alone do not make an unbounded path search
or undecidable validation step executable. This appendix does not construct
a complete certificate system for the unrestricted comparison.

The baseline and candidate must also be compared at a common accounting
point. Candidate search and rejected trials can consume resources even
when the persistent model is rolled back. Equation~\eqref{eq:growth-commit}
does not restore those resources or prove that the entire search trajectory
is cost-free or non-regressive from an earlier accounting origin.

\subsection{Conditional Liveness in an Enumerable Sector}
\label{app:growth-liveness}
Weak non-regression is a safety property. A learner can retain its incumbent
forever by rejecting every proposal, so Theorem~\ref{thm:growth} does not
imply discovery. Consider a registered candidate family
$\mathcal N(M,r)=\{N_1,N_2,\ldots\}$ for a persistent residual.
For the following result, a completed fair evaluation means that every
finite-index candidate is eventually evaluated to completion with its
required admissible evidence available. A provisional evaluation before
that evidence arrives does not discharge this condition.

\begin{proposition}[Conditional discovery of a lawful closure]
\label{prop:growth-liveness}
Assume that the residual remains identifiable and future relevant; the
required admissible evidence eventually becomes available; a finite-index
candidate $N_{j^\star}$ is a lawful closure when evaluated; evaluations
are fair in the completed sense just specified; the audit accepts every
genuinely lawful candidate in this family; and sufficient resources remain
through evaluation of $N_{j^\star}$. Then a lawful closure is eventually
found and passes that audit, becoming available for commit subject to the
full commit conditions.
\end{proposition}
\begin{proof}
Fairness gives a completed evaluation of $N_{j^\star}$. The candidate is
lawful by assumption, and audit completeness for lawful candidates makes
it pass. It is therefore found and available for commit subject to the
stated commit conditions.
\end{proof}

Audit soundness, used for safety, and audit completeness on the lawful
family, used here for discovery, are separate assumptions. In particular,
passing a current-closure audit does not establish the additional dominance
gate in Equation~\eqref{eq:growth-commit}. That gate and the persistence
policy must also be covered before claiming eventual acceptance under T3.
The proposition asserts no strict frontier improvement unless the found
candidate additionally has a strict growth witness.

An infinite enumeration with fairness supplies no uniform finite-time or
fixed-budget bound. A bounded enumeration claim would require a finite
candidate bound, terminating evaluation and certification, appropriate
evidence availability, and a corresponding resource bound. No such
quantitative bound is derived here. The result also does not decide that
no closure exists or imply infinite strict growth. Failure of its
assumptions can leave the learner open or inert; it does not automatically
produce a verifiable non-closure certificate.

\subsection{Compilation, Core, and Full Cost Accounting}
\label{app:growth-compilation}
Accumulating low-level responsibilities $A_1d=0,A_2d=0,\ldots$ can reduce the
available same-representation directions. GRC proposes compilation
as a complementary route to structural expansion. A relation becomes a
compilation candidate when it is repeatedly successful, reopenable,
dynamically sufficient, and actually reused, and its replacement by a
primitive respects the contract:
\[
 \rho^{\rm certified}\xrightarrow{\rm Compile}p_\rho.
\]
Here $p_\rho$ is the installed reusable primitive, distinct from an entire
quotient state space. When lawful, Core can then factor, merge, prune, retire, or remove and
reoptimize redundant realization structure. The intended benefit is to move
repeated causal work into a reusable primitive, while preserving the
information needed for later reopening and revision.

The two proposed growth mechanisms have different success conditions.
Carry supports a growth claim only when a certified successor enlarges the
lawful future domain without other regression. Compile supports a growth
claim only when a certified successor preserves the frontier and reduces
optimal cost for a future demand that reuses the compiled relation.
These roles do not by themselves prove either outcome for a specified
learner family. Each application must supply a construction or certificate
establishing its claimed improvement.

\paragraph{Carry as domain growth.}
If a registered residual satisfies $c^\star_\Xi(r\mid M)=+\infty$ but
$c^\star_\Xi(r\mid M^+)<+\infty$, and the successor also preserves the
old semantics and entire frontier, then
\[
\mathcal V_\Xi(M^+)\supsetneq\mathcal V_\Xi(M),\qquad
M^+\succ_\Xi M.
\]
This is the domain-growth clause of the definition. Enabling one new demand
while damaging existing capabilities does not satisfy it.

\paragraph{Compile as cost growth.}
Let a recurrent relation be replaced by $p_\rho$, with registered semantics,
dynamic sufficiency, and future viability preserved. To certify growth in
the order above, the successor must also preserve every old optimal-cost
bound and strictly reduce $c^\star_\Xi(r\mid M)$ for some previously viable
registered residual. Mere preservation of feasibility does not establish
the cost bounds. Likewise, making one particular reused path cheaper does
not by itself show that the infimum over all paths decreases.
The path-level benefit is therefore a candidate mechanism whose frontier
and optimal-cost conditions still require a certificate.

\paragraph{Core and closure capital.}
A Core simplification is lawful for this comparison only if
$M_{\rm core}\succeq_\Xi M_{\rm pre\text{-}core}$. Deletion alone is not
the defining property. Merge, factorization, reoptimization, pruning, or
retirement may simplify an implementation while retaining that comparison;
a strict optimal-cost improvement supplies cost growth. Carry and
Compile/Core thus address different coordinates of closure capability:
the viable domain and its costs. The interpretation of past stable closure
as future learning capital is conditional on these certified improvements.
Neither the existence of an eligible compilation nor its discovery follows
from stable--plastic learning alone; the capitalization property in
Appendix~\ref{app:representation} remains an additional requirement.

The cost comparison must fix its accounting origin and resource scope.
Search, rejected trials, compilation, maintenance, and reopening must be
recorded wherever the declared full-cost contract assigns them. A reduction
in continuation cost from the successor is not by itself a reduction of
total lifecycle cost from before the search. Rollback of a candidate does
not refund resources already consumed. A claim that compilation releases
resources or degrees of freedom must be checked under this same accounting
and semantic contract.

\section{Restricted-Sector Correspondences and Comparison Details}
\label{app:sectors}

\subsection{Scope of a restricted-sector correspondence}
Let $\mathcal F_{\mathrm{GRC}}^{\Xi}$ denote the admissible family of learner
states and transition mechanisms under contract $\Xi$.
A restricted method $A$ is contained only when an embedding
\begin{equation}
\iota_A:\mathcal F_A\hookrightarrow\mathcal F_{\mathrm{GRC}}^{\Xi}
\end{equation}
preserves the quantities relevant to the comparison: initialization,
admissible information, randomness, update behavior, and observable output.
Thus representation of an output function is weaker than recovery of the
declared learning dynamics. The correspondences below specify restricted
choices; they are not a construction of arbitrary solvers from primitive
operations.

The proof status is specific to each correspondence. The quadratic
optimization calculation below is an exact algebraic recovery under its
freezing assumptions. Decision and control correspondences additionally
require the declared state sufficiency and transition semantics. The extended
map in Section~\ref{app:sector-extended} contains explanatory structural
correspondences, not uniformly proved embeddings. None of these categories
alone establishes superiority over the methods being compared.

\subsection{Fixed-representation quadratic closure}
Restrict the persistent parameter state to $M=\theta\in\Theta$, holding the
representation, boundary, program, and writable coordinates fixed. At the
current state let $g=\nabla_\theta L(\theta)$ and choose
\begin{equation}
\Psi_\theta(d)=\frac{1}{2\eta}d^\top H d,
\qquad \eta>0,\qquad H\succ0.
\end{equation}

\begin{proposition}[Quadratic local optimization correspondence]
\label{prop:sector-quadratic}
Under these restrictions, the quadratic local closure problem
\begin{equation}
d^\star=\underset{d}{\arg\min}\,
\left\{g^\top d+\frac{1}{2\eta}d^\top H d\right\}
\end{equation}
has the unique solution $d^\star=-\eta H^{-1}g$.
\end{proposition}

\begin{proof}
The positive-definite quadratic objective is strictly convex. Its
first-order condition is $g+\eta^{-1}Hd=0$, which gives
$d^\star=-\eta H^{-1}g$.
\end{proof}

The choice $H=I$ recovers a gradient step. Choosing an appropriate
positive-definite Fisher metric gives a natural-gradient direction
\citep{boyd2004convex,amari1998natural}.
A positive-definite local Hessian or a suitable regularized approximation
motivates a Newton- or Gauss--Newton-type interpretation; the statement above
does not identify arbitrary Hessian choices with a well-posed minimum.
Proximal and mirror formulations require their respective
objectives, geometries, and regularity assumptions. In particular, the
proximal point method is not recovered simply by renaming the quadratic
linearization \citep{rockafellar1976proximal}.
These are additional comparison directions, not further recoveries proved
here.

The structural distinction is between choosing a change within a declared
coordinate family and deciding whether that family can express the required
repair. The quadratic calculation addresses the former. It supplies neither
a representation-obstruction test nor a procedure for constructing a new
representation.

\subsection{Constrained writable coordinates}
For a declared local edit set $\mathcal D(\theta)$, the corresponding
restricted problem is
\begin{equation}
 d^\star\in\underset{d\in\mathcal D(\theta)}{\arg\min}
 \left\{g^\top d+\frac{1}{2\eta}d^\top Hd\right\}.
 \label{eq:sector-constrained-edit}
\end{equation}
This specifies constrained optimization in the chosen coordinates when
the minimum exists. The quadratic assumptions alone do not supply existence
for an arbitrary edit set. Identifying a particular projected method also
requires its projection geometry and update rule. An accepted update is
Transformation learning only when it preserves the representation class,
$[C_{\theta+d^\star}]=[C_\theta]$, under the fixed contract and complete
persistent-state description. Numerical optimization inside this set cannot
remove a certified obstruction applying to the whole retained
representation family.

\subsection{Differentiable families and additive residual blocks}
Fix an architecture $\mathcal P$, a boundary, and writable parameters
$\theta\in\Theta$. The resulting family is
\begin{equation}
\mathcal F_{\mathcal P}=\{M_\theta:\theta\in\Theta\}.
\end{equation}
Changes in $\theta$ may alter both relations and latent representations
$z_\theta(x)$. A restriction to this family therefore fixes the outer machine
grammar, not the learned feature values. Its interpretation as a GRC sector
is conditional on preserving the declared initialization, execution, and
learning updates within the admissible family under $\Xi$.
The inclusion $\mathcal F_{\mathcal P}\subseteq
\mathcal F_{\mathrm{GRC}}^\Xi$ restricts the candidate family; recovering
a particular training method further requires the embedding to intertwine
its declared updates. A fixed architecture alone does not imply a fixed
quotient $K_{M_\theta}$ or make all parameter learning Transformation.

The identity-shortcut additive core
\begin{equation}
h_{\ell+1}=h_\ell+F_\ell(h_\ell;\theta_\ell)
\end{equation}
prescribes the block interface, reference input, and form of the correction
\citep{he2016resnet}.
Here a fixed owner means a specified interface at which the correction is
applied; it does not mean a constant hidden state or an absence of feature
learning. The displayed expression concerns the forward block. A complete
ResNet realization also includes its actual shortcut choices, nonlinearities,
and training state. Consequently this expression alone is not an embedding
of the complete training dynamics.

If a required distinction cannot be expressed anywhere in the declared
family, further optimization confined to that family cannot supply it.
Revising the representation, owner structure, memory, or program is then a
candidate completion. This observation does not assert that every residual
reveals such an obstruction, or that a legal revision is always available.

\subsection{Fixed-state sequential decision mechanisms}
Suppose a declared control abstraction $s=C(x)$ is sufficient for the registered control future, with
specified actions, transition semantics, rewards or costs, and control
boundary; here $x$ may be an interaction history. The Bellman consistency
relation and its residual are
\begin{equation}
V=\mathcal T V,
\qquad
r_{\mathrm{Bellman}}=\mathcal T V-V.
\end{equation}
This residual concerns a relation within an already declared decision model
\citep{bellman1952dynamic}. Value, policy, and controller changes on that
fixed state space provide the restricted correspondence. Recovering any
particular learning algorithm additionally requires its sampling,
estimation, and update semantics; the fixed-point equation alone does not
provide them. Value, policy, or controller changes count as Transformation
when their persistent transitions preserve $[C_M]$ under the fixed contract
and complete learner-state description. A fixed state-vector interface by
itself does not establish that quotient condition.

The representation issue arises for a proposed abstraction $C$ when
\begin{equation}
C(x_1)=C(x_2),
\qquad
\operatorname{Future}_\Xi(x_1)\not\simeq_\Xi
\operatorname{Future}_\Xi(x_2).
\end{equation}
If the contract requires the distinguishing future quantity to be
represented, that distinction cannot be expressed as a function of the
shared state alone. This is an obstruction of the stated abstraction,
consistent with the role of future-relevant state equivalence in predictive
representations and MDP reduction
\citep{littman2001predictive,givan2003equivalence}.
Different continuations need not imply different values or optimal actions
for every objective, so the claim concerns the quantity actually registered
in $\Xi$.

\begin{proposition}[State-aliasing obstruction]
\label{prop:sector-state-aliasing}
If two detailed states have the same represented state but require
inequivalent registered continuations, no relation law receiving only
that represented state can represent both required continuations.
\end{proposition}

\begin{proof}
Suppose such a law $\rho$ existed. Writing
$C(x_1)=C(x_2)=s$, it would satisfy
\[
 \rho(C(x_1))=\rho(s)=\rho(C(x_2)).
\]
The contract instead requires inequivalent future continuations for
$x_1$ and $x_2$, a contradiction.
\end{proof}

This is an obstruction of the proposed abstraction: it means that the
claimed encoder fails the required sufficiency contract. It does not
contradict the definition of a quotient already known to preserve those
same futures. Optimizing the relation more accurately while retaining the
aliased input cannot supply the missing distinction.

State splitting, memory, belief states, recurrent states, or a revised model
state are possible responses, subject to the same admissibility contract.
For constrained control, established stability and feasibility assumptions
remain necessary to the claimed correspondence
\citep{mayne2000mpc}. No separate HJB, MPC, or complete RL algorithm embedding
is derived from the Bellman residual here.

\subsection{Completion and preservation of a lawful incumbent}
Write
\begin{equation}
A^+=\operatorname{GRCComplete}_\Xi(A),
\qquad
\mathcal F_A\subseteq\mathcal F_{A^+},
\end{equation}
for a completion that preserves the lawful realizations of $A$ while
reopening selected representation, state, boundary, owner, program,
geometry, or update-rule choices. These are permitted variables, not a
requirement to revise all of them at every step.
If a comparison changes the learner--world boundary or contract, its
information rights, transport, and baseline must be restated. Such a change
is not silently covered by the fixed-boundary internal two-mode claim.

Let $a_{\mathrm{inc}}$ be the incumbent mechanism. Only candidates meeting the
current residual, old-semantic, dynamic-sufficiency, future-viability,
resource, and safety requirements belong to $\mathcal A_t^{\mathrm{legal}}$.
The selection specification is
\begin{equation}
a^\star\in
\underset{a\in\mathcal A_t^{\mathrm{legal}}}{\arg\min}\,
C_{\mathrm{full}}(a).
\end{equation}
The incumbent is included whenever it remains lawful. A merely available
incumbent is not automatically a feasible fallback if it fails the current
contract. The displayed specification also does not establish that a
minimizer exists or that an executable search returns one.

\subsection{Feasible-set comparison and a strict-improvement witness}
Let $J_\Xi$ be a common complete objective. Define
\begin{equation}
J_A^\star=\inf_{M\in\mathcal F_A}J_\Xi(M),
\qquad
J_{A^+}^\star=\inf_{M\in\mathcal F_{A^+}}J_\Xi(M).
\end{equation}

\begin{proposition}[Incumbent-preserving feasible-set comparison]
\label{prop:sector-feasible-comparison}
If $\mathcal F_A\subseteq\mathcal F_{A^+}$ under the same information rights,
constraints, objective, and complete cost accounting, then
\begin{equation}
J_{A^+}^\star\le J_A^\star.
\end{equation}
\end{proposition}

\begin{proof}
Every realization feasible for $A$ is feasible for $A^+$. The infimum over a
superset cannot exceed the infimum over its subset.
\end{proof}

\begin{corollary}[Conditional strict extension]
\label{cor:sector-strict-extension}
Under the same comparison conditions, if there exists
\begin{equation}
M'\in\mathcal F_{A^+}\setminus\mathcal F_A,
\qquad
J_\Xi(M')<J_A^\star,
\end{equation}
then $J_{A^+}^\star<J_A^\star$.
\end{corollary}

The witness condition is not supplied by set enlargement alone.
Operational costs include search, audit, compilation, maintenance, and
deployment. Inclusion must therefore concern the feasible realizations of
the actual comparison, after its resource restrictions and costs are
specified. This idealized comparison of infima is not a guarantee that a
particular numerical solver or structural search attains either value.

\subsection{Containment and discovery are distinct claims}
Containment recovers a declared method. Rediscovery would require a learner
not given that method to construct an operationally equivalent mechanism
from primitive operations and feedback. Retaining the ability to revise a
useful mechanism when it later becomes insufficient is a further property.
The restricted correspondences and feasible-set comparison do not solve
these search and computability questions. A recurrent successful mechanism
can be considered for compilation only under the dynamic-sufficiency and
cost conditions imposed elsewhere in the framework.

\subsection{Extended Structural Correspondences}
\label{app:sector-extended}
The following correspondences identify a mechanism's predeclared choices
and the choices a GRC completion could reopen. Except for the stated read
allocation calculation, they are structural interpretations requiring
further specifications for an exact algorithmic embedding. Changing an
implementation component is classified by its effect on $[C_M]$ and the
relation law, not by its architectural name. The two-mode classification
applies to persistent internal transitions under a fixed contract and
boundary, with the completeness assumption A3.

\paragraph{Attention as content-addressed opening.}
The following calculation gives a restricted correspondence with
attention \citep{vaswani2017attention}.
Consider read allocations $a\in\Delta_m$ over a fixed candidate set, with
finite costs $c_j$, a full-support prior allocation $\omega\in\Delta_m$,
and temperature $\tau>0$. The source considers
\begin{equation}
 \min_{a\in\Delta_m}\ \sum_{j=1}^m a_jc_j
       +\tau D_{\rm KL}(a\Vert\omega).
 \label{eq:sector-attention-objective}
\end{equation}
Its Lagrangian first-order condition is
\[
 c_j+\tau\left(\log\frac{a_j}{\omega_j}+1\right)+\lambda=0,
\]
giving the normalized allocation
\begin{equation}
 a_j^\star=
 \frac{\omega_j\exp(-c_j/\tau)}
      {\sum_{i=1}^m\omega_i\exp(-c_i/\tau)}.
 \label{eq:sector-attention-allocation}
\end{equation}
For a numerical query $u\in\mathbb R^{d_k}$ and numerical keys
$z_j\in\mathbb R^{d_k}$, choosing
$c_j=-u^\top z_j/\sqrt{d_k}$ gives soft content-addressed read weights.
The numerical keys are distinct from the quotient states $k=C_Mx$;
their representation must be declared. This recovers the allocation rule
for the chosen objective, not every value-combination, positional,
normalization, or training operation of a Transformer. The interpretation
as Open also requires those reads to expose the declared relational
realizations. A completion may reconsider the candidate owners, read
representation, or read grammar.

\paragraph{Recurrent state, state-space models, and filtering.}
If two detailed histories share a candidate compression $\widetilde C(x)$
but require different represented successors under the same event, that compression is not update
sufficient. A persistent recurrence
\begin{equation}
 s_{t+1}=\Psi(s_t,o_{t+1})
 \label{eq:sector-recurrent-state}
\end{equation}
can serve as a remedy when it actually carries the missing distinction
and supports the required updates. Recurrence alone does not establish
that property. RNN, state-space, or filtering mechanisms provide possible
realizations inside a declared recurrent-state grammar; a completion may
split, merge, or revise that grammar. Neither recurrence's necessity in
every task nor sufficiency of every recurrent model is asserted.

\paragraph{Sparse experts and ownership.}
When a residual can lawfully be assigned to a small subset of existing
owners, a sparse expert mechanism provides a possible realization using
selected local relation laws. This is the source's structural analogy with
mixtures of experts under a fixed expert/router vocabulary
\citep{fedus2022switch}. It is not a
unique derivation of an MoE architecture or proof that an arbitrary routing
rule performs the required diagnosis. A completion may consider owner
creation, merging, retirement, and certified Compile/Core operations.

\paragraph{Low-rank writable families.}
A low-rank matrix update restricts edits to
\begin{equation}
 \Delta W=UV,\qquad
 U\in\mathbb R^{m\times r},\quad V\in\mathbb R^{r\times n},
 \label{eq:sector-low-rank}
\end{equation}
for a prescribed factor size $r$. This defines a constrained writable
family. Calling it a same-representation fibre additionally requires
preservation of $[C_M]$; low rank by itself does not supply that condition.
The correspondence concerns this factorized update form, not all adapter
designs. Exhaustion of a numerical search or edit budget does not by itself
certify a representation obstruction or justify automatic rank expansion.
Evidence, transport, search, and genuine expressivity limitations must be
distinguished before claiming that representation change is necessary.

\paragraph{Retrieval and tools.}
A query, retrieval, sensing action, or tool call may change available
evidence without changing persistent learner state. In that case it is
an external resolution action, not a third primitive internal learning
mode. Missing evidence can explain an apparent inability to close a
residual within the current internal response family. If the evidence is
subsequently incorporated persistently, that incorporation is subject to
the internal two-mode classification. A persistent revision of the
retrieval or tool mechanism must likewise be classified by its actual
effect on the representation and relation law; revising an interface
does not alone prove that the quotient changed.

\paragraph{Predeclared structural expansion.}
Width, node, module, or feature expansion rules specify candidate structural
edits in advance. A GRC interpretation can consider such candidates when
diagnosis supports them, while keeping a representation-obstruction
certificate distinct from a failed optimization attempt. A claim that
representation revision is necessary requires that certificate; an
architectural expansion is not automatically a change of quotient.
No unique expansion algorithm or general tractability result follows.

\paragraph{Architecture search and program synthesis.}
Searching over architectures or programs can reopen structural choices
inside a fixed search grammar. It is representation search when those
choices change $[C_M]$; some implementation changes may instead preserve
that abstraction. A completion can place the grammar, solver, or compiler
among the persistent objects available for revision. This extends the
candidate description, without establishing that unrestricted search or
grammar revision is computationally tractable.

\paragraph{Meta-learning and learned update rules.}
An initialization, adaptation rule, optimizer, or other meta-state can be
included in $M$. A meta-learning method then operates within its declared
meta-family. The recursive question is whether that family remains
sufficient: an updater can itself become an object of learning.
This interpretation does not identify every change to an updater with
Representation learning or supply a method for discovering a better one.

\paragraph{Probabilistic and variational representations.}
A probabilistic learner can fix latent variables, a generative graph,
a probability family, and a variational family, then perform inference
or learning within that ontology. Such a method occupies a restricted
sector when a semantics-preserving embedding is supplied. A completion
may reconsider the latent variables, factorization, or representation
language. This is a conditional correspondence; it is not a complete
embedding of the Free Energy Principle or active inference
\citep{friston2010freeenergy}.

\begin{table}[t]
\centering
\small
\caption{Status of the sector map. Reopened choices are candidate freedoms,
not guarantees of discovery, strict improvement, or lower complete cost.}
\label{tab:sector-extended-map}
\begin{tabular}{@{}>{\raggedright\arraybackslash}p{0.20\linewidth}
                    >{\raggedright\arraybackslash}p{0.46\linewidth}
                    >{\raggedright\arraybackslash}p{0.27\linewidth}@{}}
\toprule
Mechanism & Correspondence and status & Choices that may reopen\\
\midrule
Quadratic optimization & Exact stated quadratic-step algebra with frozen representation & Edit geometry, representation\\[0.35em]
Deep networks & Conditional containment of a declared differentiable family & Machine/program family\\[0.35em]
Residual blocks & Additive forward-map structure & Correction owner and form\\[0.35em]
Attention & Exact allocation for the stated entropic read objective & Read grammar and candidates\\[0.35em]
RNN/SSM/filtering & Conditional realization of update-sufficient state & Recurrent-state representation\\[0.35em]
Sparse experts & Structural ownership interpretation & Owner creation, merge, retirement\\[0.35em]
Low-rank updates & Constrained edits; Transformation only if $[C_M]$ is preserved & Rank and writable family\\[0.35em]
Retrieval/tools & External resolution when persistent state is unchanged & Evidence access and mechanism\\[0.35em]
NAS/synthesis & Structural search in a declared grammar & Grammar, solver, compiler\\[0.35em]
Meta-learning & Conditional learned-updater sector & Updater/meta-family\\[0.35em]
Bayes/VI/FEP & Conditional probabilistic representation sector & Latent/model ontology\\
\bottomrule
\end{tabular}
\end{table}

\clearpage
\section{Counterexamples and Impossibility Regimes}
\label{app:counterexamples}
These supplied examples separate current prediction, lawful acquisition,
future response, representation adequacy, and the scope of the declared
contract. All comparisons retain the assumptions specified below.

\subsection{Prediction Sufficiency Does Not Imply Update Sufficiency}
\label{counter:prediction-update}
Let $\Hist=\{x_1,x_2\}$ and suppose both histories require the same current
prediction. A compression with $\widetilde C(x_1)=\widetilde C(x_2)$ is
sufficient for that prediction. If the same feedback event requires
\[
 C^+U_e(x_1)=a,\qquad C^+U_e(x_2)=b,\qquad a\ne b,
\]
no map $\bar U_e$ on the shared compressed state can satisfy
$C^+U_e=\bar U_e\widetilde C$. This is the two-state obstruction of
Appendix~\ref{app:update-counterexample}; the general factorization
criterion is Proposition~\ref{prop:update-factorization}. The candidate
compression is sufficient for the current prediction, not for the full
registered future-update contract.

\subsection{Loss Reduction Does Not Imply Lawful Learning}
\label{counter:loss-stability}
Take a scalar learner state $\theta\in\R$, an old registered responsibility
$S_{\rm old}(\theta)=\theta=0$, and a current loss
$L_{\rm new}(\theta)=(\theta-1)^2$. Starting at $\theta_0=0$, the edit
$\theta_1=1/2$ gives
\[
 L_{\rm new}(\theta_1)=\tfrac14<1=L_{\rm new}(\theta_0),
 \qquad S_{\rm old}(\theta_1)=\tfrac12\ne0.
\]
The current loss improves while the old responsibility fails. Loss
reduction by itself therefore certifies neither semantic stability nor
the legality of a persistent transition.

\subsection{Present Acquisition Can Exhaust a Future Local Response}
\label{counter:future-collapse}
As in Example~\ref{future-collapse-example}, let $P=I_2$ and register the
new hard responsibility $D=\begin{bmatrix}1&0\end{bmatrix}$. Then
\[
 P_{\rm after}=P-PD^\top(DPD^\top)^\dagger DP
 =\begin{bmatrix}0&0\\0&1\end{bmatrix}.
\]
For $B_f=\begin{bmatrix}1&0\end{bmatrix}$,
$B_fP_{\rm after}B_f^\top=0$. Previous responsibilities and successful
current acquisition need not leave a nonzero same-representation
first-order response for this future probe. This example does not rule
out a lawful nonlocal or representation-changing continuation, and is
therefore distinct from an impossibility result for all future closure
paths.

\subsection{Excessive Closure Load Need Not Diagnose Representation Failure}
\label{counter:missing-evidence}
Suppose the correct relation depends on an environmental bit
$z\in\{0,1\}$ and the existing relation family can already represent both
branches $\rho(s,z)$. If the current evidence does not reveal $z$, an
internal edit based only on that evidence cannot identify which branch
to install. A registered diagnostic that treats this unresolved
evidence requirement as infeasible may report $q_T(r)>1$.

An admissible external query can reveal $z$, after which the same
representation selects an already expressible relation. Thus an
overflow report alone does not establish representation insufficiency.
The example concerns an evidence-limited diagnostic; it does not assign
a numerical value to the geometric load without its required
calibration. Evidence diagnosis must precede an inference that Carry is
necessary, as specified in Appendix~\ref{app:grade}.

\subsection{A Trivial Safe Space Eliminates Local Plasticity}
\label{counter:zero-safe-space}
If $\ker A=\{0\}$, then the only semantic-safe edit is $d=0$.
Lemma~\ref{lem:safe-inverse} gives
\[
 \im P_H=\ker A=\{0\},\qquad P_H=0,
 \qquad \Kplast=BP_HB^\top=0.
\]
Hence $P_HB^\top r=0$ for every residual: there is no nonzero safe
first-order update in the declared edit space. This is an impossibility
regime of the compatibility theorem. Remaining open, obtaining
admissible evidence, revising the contract, or undertaking a justified
representation revision are different possibilities, none guaranteed
to restore feasibility. A contract revision also changes the comparison
being made.

\subsection{Protecting Historical Realization Can Remove Useful Plasticity}
\label{counter:reconstructive-strict}
Take $H=I$ on $\R^2$. Suppose exact historical preservation allows only
$T_{\rm exact}=\operatorname{span}(e_1)$, whereas the same registered
future semantics permits $T_{\rm rec}=\operatorname{span}(e_1,e_2)$.
For $B=\begin{bmatrix}0&1\end{bmatrix}$, the corresponding response
operators are
\[
 K_{\rm plast}^{\rm exact}=0,\qquad
 K_{\rm plast}^{\rm rec}=1.
\]
Exact realization preservation admits no safe response in this direction;
semantic preservation admits one. This supplied example realizes the
strict-coupling case of Appendix~\ref{app:reconstructive-dominance}.
The extra rigidity comes from protecting implementation details absent
from the declared semantic contract.

\subsection{The Two-Mode Classification Fixes the Learner--World Boundary}
\label{counter:boundary-change}
Suppose two environmental states cannot be distinguished because a sensor
channel $z$ is unavailable. Adding an external sensor that reveals $z$
changes the boundary $\partial$. Relative to the original learner-state
description, this event is outside the fixed-boundary domain of the
internal two-mode classification.

One may instead include a persistent, learner-controlled sensor interface
in an enlarged state description. In that enlarged description the
classification applies under its stated completeness assumptions: a
quotient-changing revision is Representation Learning, and a change
within the same quotient is Transformation Learning. Merely naming an
interface change does not determine its mode. The boundary and state
description must be fixed before applying the classification.

\subsection{Semantic Identity Is Contract-Relative}
\label{counter:contract-relative}
Let $x$ and $y$ have identical responses under every context in
$\Xi_1$, so $x\sim_{\Xi_1,M}y$. Enlarge the contract to $\Xi_2$ by adding
a query $c^\star$ such that
\[
 \operatorname{Resp}_M(c^\star,x)
 \not\simeq\operatorname{Resp}_M(c^\star,y).
\]
Then $x\not\sim_{\Xi_2,M}y$. The semantic quotient is relative to the
declared future-context class. If that class is unidentified, the
intended quotient is likewise unidentified; selecting an arbitrary
equivalence relation cannot certify sufficiency for it.

These examples identify separate limitations. None licenses replacing
an information obstruction with an optimization claim, or extending a
local response calculation into an unrestricted guarantee of continual
learning.

\section{Conditional Representation: Assumptions and Construction}
\label{app:representation}

\subsection{Abstract learner and operational assumptions}
Fix a contract $\Xi$ and a learner--world boundary. An abstract learner is
described by
\[
 \mathcal L=(\mathcal X,\Learner,\mathcal E,
 \operatorname{Resp},U,\operatorname{Acc},C_{\rm full}).
\]
Here $\mathcal X$ contains detailed interaction states or histories,
$\Learner$ contains persistent learner states, and $\mathcal E$ is the
registered event family. The response $\operatorname{Resp}_M(c,x)$ is
registered for context $c\in\operatorname{Ctx}_\Xi$ and detailed state $x$.
An event continuation $U_e$ acts on the detailed state, whereas a persistent
learning transition changes $M$ to $M^+$. Their domains and any event-specific
evaluation or transport are declared separately, as in
Appendix~\ref{app:semantics}. The learner's own acceptance specification
determines $\operatorname{Acc}_\Xi(M,r)$; $C_{\rm full}(N\mid M)$ is the
complete declared transition cost. No neural, parametric, probabilistic,
or differentiable realization is imposed.

The assumptions of Theorem~\ref{thm:representation} concern the process
being represented. They are not consequences of stability and plasticity
alone.
\begin{description}
\item[R1: Persistent-state completeness.]
Every persistent learner-internal variable affecting registered future
responses or subsequent accepted updates belongs to $M$.

\item[R2: Extensional future distinguishability.]
For each fixed $M$, the registered response identity in
Equation~\eqref{eq:history-kernel-quotient} induces a well-defined
equivalence relation $\sim_{\Xi,M}$ on $\mathcal X$. Distinctions irrelevant
to every registered future context may be collapsed; distinctions that
affect such a context may not be silently identified. An approximate
discrepancy does not automatically define a transitive equivalence.

\item[R3: Conditional dynamic sufficiency.]
For events claimed to operate adequately within the current abstraction,
equivalent detailed states have equivalent represented continuations.
The declared continuation is constant on the relevant quotient fibres up
to contract equivalence and therefore admits descent on that domain.
Failure caused by a merged future-relevant distinction may instead trigger
representation revision under R5. This exception does not construct a
missing continuation map.

\item[R4: Nontrivial discrepancy response.]
If an identifiable future-relevant discrepancy is nonzero and a lawful
corrective successor exists, the learner admits a nontrivial corrective
candidate. A candidate must still pass the persistence requirements below.
This assumption supplies neither a search-time bound nor a guarantee of
finding a fully resolving successor.

\item[R5: Representation revisability.]
If no admissible relation law on the current abstraction can satisfy the
registered contract, refinement or replacement of the representation is
an admissible candidate operation. Permission to consider that operation
does not establish that a successful replacement is constructible.

\item[R6: Operational future viability.]
Accepted transitions preserve the declared future adaptive possibilities:
the registered future demands remain supported by lawful successor paths.

\item[R7: Fail-closed persistence.]
Speculative candidates need not already satisfy all persistence conditions.
Only candidates satisfying the declared semantic, dynamic, viability,
information, resource, and safety requirements may become persistent states.

\item[R8: Minimal admissible selection.]
For the accepted resolving transitions under consideration, the learner
selects a complete-cost minimizer among its accepted persistent successors
that resolve the discrepancy to tolerance. This requires a selected
minimizer where the rule is invoked. An empty resolving set or an
unattained infimum does not provide a successor. No deterministic
tie-breaking rule or unique minimizer is assumed.
\end{description}

\subsection{Future-operational quotient and relation descent}
R2 gives the fixed-$M$ quotient and its class map,
\[
 K_M=\mathcal X/{\sim_{\Xi,M}},\qquad C_M:\mathcal X\longrightarrow K_M.
\]
The quotient is determined up to admissible isomorphism by the registered
future responses. Its elements are classes of detailed states or histories;
it is not a quotient of entire machines. No historical parameter identity
enters this construction.

For an event on the applicability domain of R3, let $C_M^e$ evaluate the
detailed event successor in the declared comparison space $Z_{M,e}$.
The superscript denotes event-specific evaluation or transport, not a
newly accepted learner. Dynamic sufficiency gives
\[
 C_Mx=C_My\quad\Longrightarrow\quad
 C_M^e(U_ex)\simeq_\Xi C_M^e(U_ey).
\]
Hence there is a descended continuation, in the declared response type,
\begin{equation}
 \bar U_{M,e}:K_M\rightsquigarrow Z_{M,e},\qquad
 \bar U_{M,e}(C_Mx)\simeq_\Xi C_M^e(U_ex).
 \label{eq:representation-relation-descent}
\end{equation}
The represented result is independent of the representative $x$, up to
the registered equivalence. For a fixed unchanged evaluation domain,
$Z_{M,e}=K_M$ and $C_M^e=C_M$ recover the same-quotient case.

Collect the descended laws and the remaining persistent operational degrees
of freedom into the relation/realization description $\rho_M$. In this
construction that description retains what is needed to reproduce the
learner's registered operations, giving the pair $(C_M,\rho_M)$.
This use of a complete description must be distinguished from restricting
$\rho_M$ to a predetermined small grammar or parametric family. R1 ensures
that the variables belong to $M$; it does not prove that such a restricted
relation family can express all of them. Identifying the complete
description with a chosen Ring requires the coverage condition stated in
Assumption A3 in Appendix~\ref{app:contracts}.

\subsection{Prediction and the generalized residual}
The current relation law imposes a predicted represented continuation,
\[
 \widehat k^+=\operatorname{Pred}_{M,e}(\rho_M,k)
 =\widehat U_{M,e}(k),\qquad k=C_Mx,
\]
where the last notation is used when the selected law defines a state-level
map. Prediction need not be probabilistic or differentiable. It must have
the same declared comparison space as the realized value
$k_{\rm real}^+=C_M^e(U_ex)$. Where both values are defined, set
\begin{equation}
 \mathfrak R_{\Xi,e}(M,x)
 =\Def_\Xi\bigl(C_M^e(U_ex),\widehat U_{M,e}(C_Mx)\bigr).
 \label{eq:representation-induced-defect}
\end{equation}
This includes the fixed-domain formula with $C_M^e=C_M$.

The prediction $\widehat U_{M,e}$ may disagree with reality; it is not the
exactly descended continuation $\bar U_{M,e}$. Comparing reality with an
already exact induced continuation would give zero commutation defect,
rather than a potentially nonzero prediction residual. Dynamic sufficiency
specifies when a represented continuation is well defined; it does not
assert that the learner's current prediction is correct.

If no lawful continuation descends because a needed distinction has been
merged, record that failure as a representation obstruction. An undefined
evaluation is not assigned a numerical residual or interpreted as zero.
Thus a relation defect on an adequate representation and a failure of
representation sufficiency are distinct diagnoses. R4 applies to an
identifiable relevant defect when a lawful corrective successor exists;
the construction supplies no general obstruction-detection procedure.

\subsection{Representation classes and persistent transitions}
Representations $C_M$ and $C_N$ have the same abstraction when an admissible
isomorphism satisfies
\begin{equation}
 C_N=\phi\circ C_M,\qquad \phi:K_M\longrightarrow K_N.
 \label{eq:representation-quotient-isomorphism}
\end{equation}
Write $C_N\cong C_M$ and $\varpi_\Xi(M)=[C_M]$. The brackets denote the
encoder's class up to such relabeling, not a machine-equivalence class.
Learners can differ internally and in their relation laws while inducing
the same partition of detailed states.

Every accepted persistent transition either preserves $[C_M]$ or changes
it. When it preserves the class, its remaining operational changes are
recorded in the complete relation/realization component: this is
Transformation. When it changes the class, it revises the represented
state distinctions: this is Representation learning. Relation changes may
accompany that revision, giving a joint event rather than a third primitive
class. A no-op requires no nonzero learning component. This endpoint
classification is distinct from the connection-dependent local splitting
of Theorem~\ref{thm:two-modes}.

If the current abstraction supports no lawful relation satisfying the
contract, R5 permits a candidate $N$ with $[C_N]\ne[C_M]$. This supplies the
Representation/Carry option within the description; it guarantees neither
a successful candidate nor completion of the search.

\subsection{Inherited legality, selection, and operational equivalence}
On the resolving domain of R8, let $\operatorname{Acc}_\Xi(M,r)$ denote the
learner's own accepted successor family, including closure to the declared
tolerance. Define
\begin{equation}
 \operatorname{Law}^{\mathcal L}_\Xi(M,r)
 :=\operatorname{Acc}_\Xi(M,r).
 \label{eq:representation-inherited-law}
\end{equation}
This is the family denoted by $\operatorname{Law}_\Xi(M,r)$ in
Theorem~\ref{thm:representation}. Its preservation, dynamic sufficiency,
future viability, information rights, resource, and safety requirements
are inherited from the learner's registered acceptance semantics. No
additional GRC legality filter is inserted to alter that family. R6--R7
constrain accepted successors; they do not imply that every conceivable
lawful candidate is accepted or discovered.

The least-lawful selection correspondence is
\[
 \operatorname{LeastLawfulClose}_\Xi(M,r)
 :=\argmin_{N\in\operatorname{Law}^{\mathcal L}_\Xi(M,r)}
 C_{\rm full}(N\mid M).
\]
R8 places the learner's selected successor in this correspondence, giving
Equation~\eqref{eq:representation-lawful-selection}. The representation
retains the learner's actual selection among minimizers; the correspondence
does not license additional observed behavior by replacing that selection
with every minimizer. Minimum cost is an assumption about these transitions,
not a consequence of R1--R7.

Two descriptions are operationally equivalent under $\Xi$ when, for each
registered initial state, context, and event sequence on the declared
applicability domain, they agree up to contract equivalence on registered
responses, accepted successor semantics, information rights, and transition
legality. The constructed description retains the original detailed
continuation, response equivalence, acceptance semantics, and selection.
Thus it changes their representation, not those operational observables.
When descent is unavailable, it records the obstruction instead of
claiming a complete quotient-level continuation for that event.

\subsection{Proof of Theorem~\ref{thm:representation}}
\begin{proof}
\emph{1. Quotient existence.}
R2 makes $\sim_{\Xi,M}$ an equivalence relation, yielding $K_M$ and $C_M$.

\emph{2. Relation descent.}
On events for which the current representation is adequate, R3 makes the
represented continuation constant on quotient fibres. This gives
Equation~\eqref{eq:representation-relation-descent}. The remaining
persistent operational degrees of freedom are retained in $\rho_M$.

\emph{3. Residual construction.}
The represented realized continuation and the current law's prediction
give Equation~\eqref{eq:representation-induced-defect}. Failure of descent
is recorded as a representation obstruction, within the stated typing and
applicability restrictions.

\emph{4. Nontrivial response.}
R4 supplies a nontrivial corrective candidate for an identifiable nonzero
relevant defect whenever a lawful corrective successor exists.

\emph{5. Transition classification.}
An accepted successor either preserves $[C_M]$ or changes it. Retaining the
other persistent operational degrees of freedom in $\rho_M$ yields the
Transformation case in the former event and a Representation revision,
possibly accompanied by relation changes, in the latter.

\emph{6. Revisability.}
R5 permits representation revision when no relation law on the current
abstraction satisfies the contract. This step provides an admissible option,
not a constructive success guarantee.

\emph{7. Persistent legality.}
R6--R7 impose the learner's declared future and persistence requirements.
Equation~\eqref{eq:representation-inherited-law} retains that learner's
accepted resolving family.

\emph{8. Minimal closure.}
R8 places the selected successor among the complete-cost minimizers of that
family, giving Equation~\eqref{eq:representation-lawful-selection}.
The construction preserves the declared responses and accepted transition
semantics, establishing the stated conditional operational representation.
\end{proof}

This argument constructs descriptive objects from the stipulated process.
It does not infer relation descent beyond its applicability domain, an
effective encoding of the full quotient, or a search procedure that realizes
the assumed selection.

\subsection{Additional growth, compilation, and Core properties}
The stronger growth description additionally assumes the capability
comparison of Section~5 and Appendix~\ref{app:growth}. Relative to the
declared background transport $\mathsf T$, non-regression requires
$M^+\succeq_\Xi\mathsf TM$; a transition designated as strict Growth
also requires $M^+\succ_\Xi\mathsf TM$. The representation preserves these
properties when imposed. It does not derive them from R1--R8 or imply that
strictly improving transitions are always available.

\paragraph{R9: Recurrent capitalization.}
Theorem~\ref{thm:representation} does not imply compilation. A stable,
plastic learner could keep repeating an expensive successful relation.
The additional R9 property concerns a relation $\rho$ that recurs across
future closure paths, has been repeatedly certified, and is reopenable
from a persistent primitive. Its compilation must preserve registered
semantics, dynamic sufficiency, and the future-viability frontier while
strictly reducing the declared complete future closure cost. When these
conditions hold, R9 permits a persistent replacement
\[
 \rho^{\mathrm{recurrent,certified}}\xrightarrow{\Compile}p_\rho.
\]
The primitive $p_\rho$ is neither the quotient space $K_M$ nor a prediction
of the next world state. Reopening makes the stored primitive operational;
compilation makes certified recurrent activity reusable. These operations
have complementary temporal roles and are not asserted to be algebraic
inverses. Together they permit a recursive semiclosed description when
the stated opportunities occur, without guaranteeing their discovery or
recurrence. A compilation transition counts as Growth only under the
declared capability comparison and complete-cost accounting.

After a successful compilation, a proposed Core operation changes
$M_{\rm comp}$ to $M_{\rm core}$ only under the stated comparison
\[
 M_{\rm core}\succeq_\Xi M_{\rm comp}.
\]
Strict inequality would make this Core transition a further strict gain.
Merge, factorization, pruning, reoptimization, and retirement are possible
realization changes, not mechanisms required or constructed by the
representation theorem. Neither R9 nor the displayed acceptance condition
ensures that a qualifying Core operation can be found.

\subsection{Scope}
The result is an existence-of-operational-representation statement for a
stipulated class of processes under a fixed contract and boundary. It does
not show that every learner satisfies R1--R8, identify a unique
implementation, or establish efficient discovery, affordable representation
revision, strict growth on every event, or empirical superiority. Recursive
compilation requires the independent R9 eligibility and admission property;
Core retains its own comparison condition.

The quotient's isomorphism class does not determine a unique solver, metric,
cost model, or learning trajectory. Retaining complete persistent
operational information in a descriptive relation layer does not prove
that a restricted executable Ring is complete. A different contract can
change the quotient, lawful successor family, and capability order. The
conditional structural description therefore supplies no computational
universality or algorithmic completeness beyond its stated assumptions.

\section{Operational Intelligence, Recursive Learning, and Falsifiability}
\label{app:intelligence}

\subsection{A Profile of Future Capability and Cost}
The proposed intelligence interpretation reuses the objects of
Section~\ref{sec:growth}. Its capability profile combines a registered future
domain with the optimal complete costs of lawful closure. The contract fixes
which distinctions, information rights, semantic responsibilities, resources,
and interventions matter. Consequently, more parameters or better performance
on one fixed benchmark need not establish dominance of the whole profile.
The interpretation is that intelligence can turn unresolved, future-relevant
distinctions into reusable capacity while preserving the ability to learn
again. It does not posit a unique metaphysical definition or collapse that
profile into an unsupported universal score.

We retain the optimal-cost epigraph already defined in
Equation~\eqref{eq:growth-frontier}. Its boundary may represent an unattained
infimum. A claim of closure within an exact budget must separately provide
an admissible path at that budget; the epigraph alone does not provide one.

\subsection{Temporary, Persistent, and Compiled Outcomes}
Reasoning is temporary closure search over relations, trajectories,
decompositions, or programs. It need not change persistent knowledge.
Learning makes a residual-driven lawful transition persistent, subject to
the distinctions between learning and effective learning in
Section~\ref{sec:growth}. Compilation gives repeatedly useful closures a
reusable realization. The operational role of such memory is to preserve
conditions for reopening a capability and continuing its revision, rather
than merely storing a historical snapshot. These are different outcomes:
a successful trial is not automatically a persistent commit, and a commit
is not automatically a mature compiled primitive.

Transformation revises relations within a representation; representation
revision changes which distinctions the learner can express. A certified
relation may subsequently become a primitive supporting further relations.
Thus learning can construct and repair the state space of subsequent
learning. A successor need not change the quotient state space:
Transformation can preserve it up to isomorphism while changing its relation
law. This account is structural: it does not assert that every stored
object is compiled or that every recurrence creates a new residual. When
continued operation respects the registered obligations, the next residual
can remain zero.

\subsection{Failure Tests and Their Scope}
Formal counterexamples must satisfy the stated hypotheses. For the two-mode
classification, the proposed test is a persistent internal transition within
the fixed boundary and complete state description that admits neither a
relation change nor a representation change nor their combination. For the
local affine compatibility claim, the proposed test has
$P_HB^\top r\ne0$ but no semantic-safe direction producing positive
first-order residual reduction. For reconstructive dominance, the response
operators must be compared under the same metric and residual map; a violation
despite the stated nested safe spaces would contradict that claim.

The other checks concern the declared certificates and operational
representation. A successor worse than its transported baseline requires
examining audit soundness and the dominance gate. A counterexample to the
conditional representation result must satisfy its assumptions while lacking
the claimed operational representation. These are proof and implementation
obligations, not additional theorems established by this discussion.

Empirical failures can also expose incorrect contracts, unsound audits,
missing evidence, or excessive complete costs. Such failures need to be
distinguished from counterexamples satisfying a theorem's hypotheses.
Neither preservation nor containment alone supplies discovery, affordable
search, or automatic compilation. The substantive question is whether a
specified learner can certify and reuse closures so that acquired stability
supports its next lawful adaptation.

\section{Related Work and Positioning}
\label{app:related-work}

\paragraph{Predictive state and behavioral abstraction.}
Predictive state representations use action-conditioned future tests and
provide recursive state updates \citep{littman2001predictive}. Computational mechanics
constructs causal states from predictive equivalence
\citep{shalizi2001computational}; MDP equivalence supports behavior-preserving
state reduction \citep{givan2003equivalence}.
Consequently, a future-operational quotient, minimal predictive state, or
commuting update equation alone cannot establish GRC's novelty. Here the
registered operations include changes to the learner's persistent knowledge
and update state. A candidate encoding must be checked against those
operations and the declared background. The distinction between a current
prediction and sufficient information for a subsequent learning update must
not be confused with claiming that predictive-state methods lack dynamics.

\paragraph{Continual learning and retained plasticity.}
EWC penalizes changes to important parameters
\citep{kirkpatrick2017ewc}, whereas OGD uses gradient projections to reduce
interference with previous predictions \citep{farajtabar2020ogd}.
These are precedents for seeking useful updates under preservation
requirements. Loss of plasticity is also an established problem, with
continual backpropagation evaluated as a means of sustaining learning ability
\citep{dohare2024plasticity}. More recently, FIRE studies constrained
reinitialization using weight proximity and isometry, and SplitLoRA studies
gradient-space partitioning for continual low-rank adaptation
\citep{han2026fire,qiu2026splitlora}.
GRC's question is how preservation, available responses, and registered
future learning are specified together. Its affine response geometry does
not itself guarantee nonlinear endpoint preservation, nor demonstrate
superiority to these mechanisms.

\paragraph{Optimization geometry.}
Natural gradient and proximal optimization already organize updates through
non-Euclidean geometry or regularized subproblems
\citep{amari1998natural,rockafellar1976proximal}.
GRC uses these tools after an admissible edit family has been specified.
The upstream question is whether that family can express the required
repair. The quadratic calculation in Section~\ref{sec:sectors} recovers a
restricted metric step, not every Newton, mirror, or proximal algorithm.
Changing the objective, constraints, or geometry requires the corresponding
assumptions; naming a familiar solver does not provide its embedding.

\paragraph{Residual networks and learned representations.}
Residual networks learn corrections relative to block inputs
\citep{he2016resnet}. The additive core is a useful example of a prescribed
interface and correction form. It is a forward computation, not a complete
learning rule. Neural training can alter latent representations within its
declared family. GRC asks when a required repair lies outside that family,
without denying feature learning or treating every residual as evidence
that a new structure is necessary.

\paragraph{Learning the learning mechanism.}
MAML trains initializations for subsequent adaptation
\citep{finn2017maml}; learned optimizers parameterize update rules
\citep{andrychowicz2016learn}; DARTS relaxes architectural choices into a
differentiable search space \citep{liu2019darts}.
DreamCoder learns reusable program abstractions and a search model
\citep{ellis2021dreamcoder}. These precedents rule out claiming that learning
the solver or construction language is new by itself. GRC instead places
candidate revisions under common responsibility and continuation contracts.
The admissible grammar, information available for diagnosing its limits, and
cost of finding a revision must remain explicit.

\paragraph{Sequential decisions and control.}
Bellman recursion and constrained control already relate present decisions
to future consequences \citep{bellman1952dynamic,mayne2000mpc}.
A state abstraction can fail to preserve a required distinction even when
optimization within that abstraction succeeds. This motivates separating
state revision from relation updates. It does not establish that all
different futures require different optimal actions, or that a Bellman
fixed-point correspondence recovers every sampling-based RL algorithm.

\paragraph{Free-energy formulations.}
The free-energy principle relates perception, action, and learning through
generative probabilistic models \citep{friston2010freeenergy}.
GRC uses a contract-relative continuation discrepancy without requiring a
single probabilistic representation language. This difference in formulation
is not a proof of containment or superiority. Relating a particular
active-inference system to GRC requires its actual semantics, objectives,
and updates to be mapped explicitly.

\paragraph{Scope of the proposed contribution.}
The proposed connection concerns what changes, what remains recoverable,
and which future learning remains possible. Each mathematical statement
addresses a declared setting. Containment, completion, and strict extension
have distinct proof obligations: strict gain needs a lawful improving
witness under matched information and complete costs. The comparisons above
acknowledge relevant precedents; they do not establish that every existing
method lacks this connection or that the literature's novelty boundaries
have been exhausted.

\end{document}